\documentclass[]{article}
\usepackage[utf8]{inputenc}
\usepackage{epsfig,subfig}
\usepackage{amsmath,mathrsfs}
\usepackage{amssymb}
\usepackage{amsfonts}
\usepackage{mathtools}
\usepackage{bbm}
\usepackage{epstopdf}
\usepackage{ifthen}
\usepackage{bbm}
\usepackage{lipsum}

\usepackage{amssymb,amsthm}
\usepackage[nomessages]{fp}
\usepackage{bbm}
\usepackage{mathrsfs,bm,enumerate}
\usepackage[dvipsnames]{xcolor}
\usepackage{graphicx,color}
\usepackage{caption}
\usepackage[T1]{fontenc}
\usepackage{amsthm}
\usepackage{mathtools}
\usepackage{bookmark}
\usepackage{latexsym}
\usepackage{sidecap}
\usepackage{xcolor}
\usepackage{listings}
\usepackage{multicol}
\usepackage[export]{adjustbox}

\usepackage{xparse}

\NewDocumentCommand{\codeword}{v}{%
\texttt{\textcolor{blue}{#1}}%
}

\DeclareMathOperator*{\argmin}{\arg\min}

\newcommand*{\V}[1]{{\mathbf{#1}}}  
\newcommand*{\Spc}[1]{{\mathcal{#1}}}  

\newcommand{\fcano}{f_{\mathrm{cano}}}

\renewcommand{\[}{\begin{equation}}
\renewcommand*{\]}[1]{\label{eq:#1}\end{equation}}

\newcommand{\ie}{{\em i.e., }}
\newcommand{\eg}{{\em e.g., }}

\newtheorem{definition}{Definition}
\newtheorem{proposition}{Proposition}

\newtheorem{lemma}{Lemma}
\newtheorem{theorem}{Theorem}

\usepackage{url}
\usepackage{hyperref}

\def\BibTeX{{\rm B\kern-.05em{\sc i\kern-.025em b}\kern-.08em T\kern-.1667em\lower.7ex\hbox{E}\kern-.125emX}}
\title{Sparsest Univariate Learning Models Under Lipschitz Constraint}
\author{  Shayan Aziznejad, Thomas Debarre, and Michael Unser,
{\it Fellow, IEEE}
\thanks{This work was supported in part by the European Research Council (ERC Project FunLearn) under Grant 101020573 and in part by the Swiss National Science Foundation, Grant 200020\_184646/1.}
\thanks{The authors are with the Biomedical Imaging Group, \'Ecole polytechnique f\'ed\'erale de Lausanne, 1015 Lausanne, Switzerland (e-mail:  shayan.aziznejad@epfl.ch; thomas.debarre@epfl.ch; michael.unser@epfl.ch). Shayan Aziznejad and Thomas Debarre contributed equally to this work.}}

\begin{document}
 
\maketitle

\begin{abstract}
Beside the minimization of the prediction error, two of the most desirable properties of a regression scheme are {\it stability} and {\it interpretability}. Driven by these principles, we propose continuous-domain formulations for one-dimensional regression problems. In our first approach, we use the Lipschitz constant as a regularizer, which results in an implicit tuning of the overall robustness of the learned mapping. In our second approach, we control the Lipschitz constant explicitly using a user-defined upper-bound and make use of a sparsity-promoting regularizer to favor simpler (and, hence, more interpretable) solutions. The theoretical study of the latter formulation is motivated in part by its equivalence, which we prove, with the training of a Lipschitz-constrained two-layer univariate neural network with rectified linear unit (ReLU) activations and weight decay. By proving representer theorems, we show that both problems admit global minimizers that are continuous and piecewise-linear (CPWL) functions. Moreover, we propose efficient algorithms that find the sparsest solution of each problem: the CPWL mapping with the least number of linear regions. Finally, we illustrate numerically the outcome of our formulations. 
\end{abstract}

{\bf Keywords:}
Robust learning, sparsity, Lipschitz regularity, continuous and piecewise linear functions, representer theorems. 

\section{Introduction}
The goal of a regression model is to learn a mapping $f:\mathcal{X}\rightarrow\mathcal{Y}$ from a collection of data points $(x_m,y_m)\in \mathcal{X}\times \mathcal{Y}, \ m=1,\ldots,M,$ such that $y_m\approx f(x_m)$, while avoiding the problem of overfitting  \cite{wahba1990spline,gyorfi2006distribution,hastie2009overview}. Here, $\mathcal{X}$ denotes the input domain and $\mathcal{Y}$ is the set of possible outcomes. A common way of carrying out this task is to solve a minimization problem of the form 
\begin{equation}\label{Eq:GeneralSupervisedLearning}
\min_{f\in \mathcal{F}} \left( \sum_{m=1}^M { E}\left(f(x_m),y_m\right) + \mathcal{R}(f)\right),
\end{equation}
 where $\mathcal{F}$ is the underlying search space, the convex loss function  ${ E}:\mathcal{Y}\times\mathcal{Y}\rightarrow\mathbb{R}_{\geq 0}$ enforces the consistency of the learned mapping with the given data points, and the regularization functional  $\mathcal{R}:\mathcal{F}\rightarrow \mathbb{R}_{\geq 0}$ injects prior knowledge on the form of the mapping $f$, which is designed to alleviate the problem of overfitting.

\subsection{Nonparametric regression}
In some cases, the optimization can be performed over an infinite-dimensional function space. A prominent example is the family of reproducing-kernel Hilbert spaces (RKHS) $\mathcal{F}=\mathcal{H}(\mathbb{R}^d)$, $\mathcal{X}= \mathbb{R}^d$, $\mathcal{Y}=\mathbb{R}$ \cite{poggio1990networks,poggio1990regularization}, in which the regression problem is formulated as
\begin{equation}\label{Eq:RKHS}
\min_{f\in \mathcal{H}(\mathbb{R}^d)} \left(\sum_{m=1}^M { E}\left(f(\boldsymbol{x}_m),y_m\right) + \lambda \|f\|_{\mathcal{H}}^2\right).
\end{equation}
The fundamental result in RKHS theory is the kernel representer theorem \cite{kimeldorf1971some,scholkopf2001generalized}, which states that the unique solution of \eqref{Eq:RKHS} admits the kernel expansion 
\begin{equation}\label{Eq:KernelExpansionL2}
f(\cdot) = \sum_{m=1}^M a_m {\rm k}(\cdot,\boldsymbol{x}_m), 
\end{equation}
where ${\rm k}:\mathbb{R}^d\times\mathbb{R}^d\rightarrow\mathbb{R}$ is the unique reproducing kernel of $\mathcal{H}(\mathbb{R}^d)$ and  $a_m\in\mathbb{R},m=1,\ldots,M,$ are learnable parameters. The expansion \eqref{Eq:KernelExpansionL2} allows one to recast the infinite-dimensional problem~\eqref{Eq:RKHS} into a finite-dimensional one and to use standard computational tools of convex analysis to solve it. Many classical kernel-based schemes are based on this approach, including support-vector machines and radial-basis functions \cite{scholkopf2001learning,evgeniou2000regularization,steinwart2008support}.

\subsection{Parametric regression}
In cases when \eqref{Eq:GeneralSupervisedLearning} cannot be recast as a finite-dimensional optimization problem, another common approach is to restrict the search space $\Spc F$ to a subspace that admits a parametric representation. This approach is used in deep neural networks (DNNs), which have become a prominent tool in machine learning and data science in recent years \cite{LeCun2015,Goodfellow2016}. They outperform classical kernel-based methods for various image-processing tasks. In particular, they have become  state-of-the-art for image classification \cite{Krizhevsky2012}, inverse problems \cite{jin2017deep}, and image segmentation \cite{Ronneberger2015}. However, most published works are empirical, and the outstanding performance of DNNs is yet to be fully understood. To this end, many recent works are directed towards studying DNNs from a theoretical perspective. Unsurprisingly, {\it stability} and {\it interpretability}, which are key principles in machine learning, play a central role in these works.
For example, the stability of state-of-the-art deep-learning-based methods has been dramatically challenged in image classification \cite{moosavi2016deepfool,fawzi2017robustness} and image reconstruction \cite{antun2020instabilities}. Attempts have also been made to understand and interpret DNNs from different perspectives, such as rate-distortion theory \cite{macdonald2019rate,heiss2020distribution}). However, the community is still far from reaching a global understanding and these questions are still active areas of research.

\subsection{Our Contributions}
In this paper, we introduce two variational formulations for regressing one-dimensional data that favor ``stable'' and ``simple'' regression models. Similar to RKHS theory, the latter are nonparametric continuous-domain problems in the sense that $\mathcal{F}$ in \eqref{Eq:GeneralSupervisedLearning} is an infinite-dimensional function space. Inspired by the stability principle, we focus on the development of regression schemes with controlled Lipschitz regularity. This is motivated by the observation that many analyses in deep learning require assumptions on the Lipschitz constant of the learned mapping \cite{arjovsky2017wasserstein,bora2017compressed,neyshabur2017exploring}. Likewise, when  a trainable module is inserted into an iterative-reconstruction framework, the rate of convergence of the overall scheme often depends on the Lipschitz constant of this module \cite{gupta2018cnn,zhang2017beyond,sun2020block,liu2020rare,wu2020simba,bohra2021learning}.

In our first formulation, we use the Lipschitz constant of the learned mapping as a regularization term. Specifically, we consider the minimization problem 
\begin{equation}\label{Eq:LipRegIntro}
\min_{f\in{\rm Lip}(\mathbb{R})} \left( \sum_{m=1}^M { E}\left(f(x_m),y_m\right) +  \lambda L(f) \right),
\end{equation}
where ${\rm Lip}(\mathbb{R})$ is the space of Lipschitz-continuous real functions and $L(f)$ denotes the Lipschitz constant of $f\in {\rm Lip}(\mathbb{R})$. In this formulation, one can implicitly control the Lipschitz regularity of the learned function by varying the regularization parameter $\lambda$. We prove a representer theorem that characterizes the solution set of \eqref{Eq:LipRegIntro}. In particular, we prove that the global minimum is achieved by a continuous and piecewise-linear (CPWL) mapping. Next, motivated by the simplicity principle, we find the mapping with the minimal number of linear regions. Note that many previous works study problems similar to \eqref{Eq:LipRegIntro} in more general settings, typically using  the Lipschitz constant of the $n$th derivative  $\mathcal{R}(f)=L(f^{(n)})$  with $n \geq 0$  as the  regularization term  \cite{mcclure1975perfect,karlin1975interpolation,de1975small,micchelli1976optimal,de1976best,pinkus1988smoothest}. More recently, \cite{von2004distance} has studied the classification problem over metric spaces and derived a parametric form for a solution of this problem. Our objective is to complement this interesting line of research by focusing more on computational aspects of \eqref{Eq:LipRegIntro} ({\it e.g.}, finding the sparsest CPWL solution) and by providing an in-depth analysis of the $n=0$ case which is related to second-order total-variation minimization.

In the second scenario, we explictly control the Lipschitz constant of the learned mapping by imposing a hard constraint. Inspired by the theoretical insights of the first problem, we add a second-order total-variation (TV) regularization term that is known to promote sparse CPWL functions \cite{unser2019Deepspline,debarre2020sparsest}. This leads to the minimization problem 
\begin{align}
 \min_{f\in{\rm BV}^{(2)}(\mathbb{R})} &\left( \sum_{m=1}^M { E}\left(f(x_m),y_m\right) +  \lambda {\rm TV}^{(2)}(f) \right), 
\nonumber\\ &\text{s.t.} \quad L(f)\leq \overline{L}, \label{Eq:LipCnstIntro}
\end{align}
where ${\rm BV}^{(2)}(\mathbb{R})$ is the space of functions with bounded second-order TV and $\overline{L}$ is the user-defined upper-bound for the desired Lipschitz regularity of the learned mapping. The interesting aspect of \eqref{Eq:LipCnstIntro} is that the simplicity and stability of the learned mapping can be adjusted by tuning the parameters $\lambda>0$ and $\overline{L}>0$, respectively. In this case as well, we prove a representer theorem  which guarantees the existence of CPWL solutions. Again, we propose an algorithm to find the sparsest CPWL solution.

\subsection{Connection to Neural Networks}
Another major motivation for this work is to further elucidate the tight connection between CPWL functions and neural networks. It is well known that the input-output mapping of any feed-forward DNN with linear spline (\eg the rectified linear unit, also known as ReLU) activations is a CPWL function \cite{pascanu2013number,Montufar2014}. Moreover, any CPWL function can be represented {\it exactly} by a DNN with linear-spline activations \cite{arora2016understanding}. This establishes a direct link with spline theory, as first highlighted by Poggio {\it et al.} \cite{Poggio2015} and then further explored in various works \cite{unser2019Deepspline,balestriero2020mad,parhi2020role,parhi2021banach,parhi2021kinds}.  

When it comes to shallow networks, the connection with our framework becomes even more explicit. It is well known in the literature that the standard training (\ie with weight decay) of a two-layer univariate ReLU network is equivalent to solving a TV-based variational problem such as~\eqref{Eq:LipCnstIntro} without the Lipschitz constraint \cite{savarese2019how, parhi2020role}. As we demonstrate, these results can be readily extended to prove the equivalence between the training of a Lipschitz-constrained two-layer univariate ReLU network and our formulation~\eqref{Eq:LipCnstIntro}. Our description of the solution set of Problem~\eqref{Eq:LipCnstIntro} thus provides insights into the training of Lipschitz-aware neural networks.

\subsection{Outline}
The  paper is organized as follows: we review the required mathematical background in Section \ref{sec:prelim}. In Section \ref{sec:Proposals}, we introduce our  supervised-learning formulations and we state their corresponding representer theorems. We then propose our algorithms for finding the corresponding sparsest CPWL solution in Section \ref{sec:Algo}. Finally, we provide numerical illustrations and discussions in Section \ref{sec:Numerical}.

\section{Mathematical Preliminaries}\label{sec:prelim}

\subsection{Weak Derivatives}
Schwartz' space of smooth and compactly supported test functions is denoted by $\mathcal{D}(\mathbb{R})$. It is known that the $n$th-order derivative is a  continuous mapping   over $\mathcal{D}(\mathbb{R})$, which we denote as  ${\rm D}^n:\mathcal{D}(\mathbb{R}) \rightarrow \mathcal{D}(\mathbb{R})$ \cite{schwartz1966theorie}. By duality, this allows one to extend the derivative operator to the whole class $\mathcal{D}'(\mathbb{R})$ of distributions.  The extended operator is called the $n$th-order weak derivative and will be denoted by   ${\rm D}^n:\mathcal{D}'(\mathbb{R}) \rightarrow \mathcal{D}'(\mathbb{R})$. For any $w\in \mathcal{D}'(\mathbb{R})$, the distribution ${\rm D}^n\{w\}\in \mathcal{D}'(\mathbb{R})$ is defined via its action on a generic test function $\varphi\in\mathcal{D}(\mathbb{R})$ as 
$
\langle{\rm D}^n w, \varphi \rangle = (-1)^n \langle w, {\rm D}^n \varphi\rangle$.
 The fundamental property is that the weak derivative of any Schwartz test function $\varphi \in \mathcal{D}(\mathbb{R})\subseteq \mathcal{D}'(\mathbb{R})$ is well-defined and coincides with the classical notion of derivative   (see \cite[Section 3.3.2.]{unser2014introduction} for more details on the extension by duality). 
 
 \subsection{Banach Spaces}
A Banach space is a normed topological vector space that is complete in its norm topology. The prototypical examples of Banach spaces are  $L_p(\mathbb{R})$ for   $p\in[1,+\infty]$ which are the spaces of Lebesgue measurable functions with finite $L_p$ norm. For $p\neq +\infty$, this reads as 
\begin{equation}\label{Eq:LpSpace}
L_p(\mathbb{R}) = \left\{f:\mathbb{R}\rightarrow\mathbb{R} \text{ measurable: } \|f\|_{L_p} <+\infty  \right\},
\end{equation}
where $\|f\|_{L_p}= \left(\int_{\mathbb{R}} |f(x)|^p {\rm d}x \right)^{\frac{1}{p}}$.
Alternatively, one can define $L_p(\mathbb{R}) = \overline{(\mathcal{D}(\mathbb{R}),\|\cdot\|_{L_p})}$  as the completion of $\mathcal{D}(\mathbb{R})$ with respect to the $L_p$ norm for $p \in [1,+\infty)$.  The case $p=+\infty$ is  particular. Indeed, the $L_\infty$ norm is defined as  $\|f\|_{L_\infty}= {\rm ess} \sup_{x\in\mathbb{R}} |f(x)|$,
 where the essential supremum  extracts an upper-bound that is valid almost everywhere. Contrarily   to the other $L_p$ spaces, the space $\mathcal{D}(\mathbb{R})$ is not dense in $L_{\infty}(\mathbb{R})$; in fact, the completion of $\mathcal{D}(\mathbb{R})$ with respect to the $L_\infty$ norm is the space $\mathcal{C}_0(\mathbb{R})$ of continuous functions that vanish at infinity \cite{rudin1991functional}. 

Finally, we denote  the space  of bounded Radon measures by $\mathcal{M}(\mathbb{R})$. Following the Riesz-Markov theorem, we view $\mathcal{M}(\mathbb{R})$ as the continuous dual of  $\mathcal{C}_0(\mathbb{R})$. This allows us to define the total-variation norm over this space as \cite{rudin2006real}
\begin{equation}\label{Eq:TV}
\|w\|_{\mathcal{M}} = \sup_{\substack{\varphi\in \mathcal{C}_0(\mathbb{R})\\ \|\varphi\|_{L_\infty}=1}} |\langle w,\varphi \rangle|=\sup_{\substack{\varphi\in \mathcal{D}(\mathbb{R})\\ \|\varphi\|_{L_\infty}=1}} |\langle w,\varphi \rangle|,
\end{equation}
where the last equality   follows from the denseness of $\mathcal{D}(\mathbb{R})$ in $\mathcal{C}_0(\mathbb{R})$. Interestingly, the total-variation norm is a generalization of the $L_1$ norm. In fact, the space $L_1(\mathbb{R})$ is included in $\mathcal{M}(\mathbb{R})$ and, for any function $f\in L_1(\mathbb{R})$, we have that $\|f\|_{L_1}  = \|f\|_{\mathcal{M}}$. Moreover, the space $\mathcal{M}(\mathbb{R})$ contains shifted Dirac impulses with $\|\delta(\cdot - x_0)\|_{\mathcal{M}}=1$. Finally,  for any absolutely summable sequence  $\boldsymbol{a}= (a_n)\in \ell_1(\mathbb{Z})$ and distinct locations $x_n,n\in\mathbb{Z}$, we have that
\begin{equation}
\label{eq:TV_norm_diracs}
w_{\boldsymbol{a}} = \sum_{n\in\mathbb{Z}} a_n \delta(\cdot-x_n) \in \mathcal{M}(\mathbb{R})  \quad \text{and} \quad  \|w_{\boldsymbol{a}}\|_{\mathcal{M}} = \|\boldsymbol{a}\|_{\ell_1}.
\end{equation}
This property   establishes a tight link between the total-variation norm and the discrete $\ell_1$ norm which is  known to promote sparsity and is the  key element  in the field of compressed sensing \cite{donoho2006compressed,candes2006robust,eldar2012compressed}. This enabled researchers to interpret the total-variation norm as a sparsity-promoting norm in the continuous domain. Since then, additional connections have been drawn between optimization problems that involve the total-variation norm and many areas of research such as super resolution \cite{bredies2013inverse,candes2014towards,duval2015exact}, kernel methods,  \cite{aziznejad2019multikernel,bach2017breaking}, and splines \cite{unser2017splines,debarre2019hybrid,flinth2019exact,bredies2020sparsity,parhi2021banach}.  The computational aspects of this framework have also been investigated, leading to the development of practical algorithms  in various settings  \cite{denoyelle2019sliding,Debarre2019Bspline,simeoni2021functional}.

\subsection{Lipschitz Constant} We denote by ${\rm Lip}(\mathbb{R})$, the space of Lipschitz-continuous functions   $f:\mathbb{R}\rightarrow \mathbb{R}$ with a finite Lipschitz constant, satisfying
\begin{equation}\label{Eq:LipCnst}
    { L}(f) = \sup_{x_1\neq x_2}\frac{ \left|f(x_1)-f(x_2)\right|}{ \left|x_1-x_2\right|} <+\infty.
\end{equation}
Following   Rademacher's theorem, any Lipschitz-continuous function $f\in {\rm Lip}(\mathbb{R})$ is  differentiable  almost everywhere with a measurable and  essentially bounded derivative. The Lipschitz constant of the function then corresponds to the essential supremum of its derivative, so that
\begin{equation}\label{Eq:LipSupD}
    { L}(f) = \|{\rm D}\{f\}\|_{L_{\infty}} = \mathrm{ess} \sup_{x\in \mathbb{R}} |f'(x)|.
\end{equation} 
 Conversely, any distribution $f\in \mathcal{D}'(\mathbb{R})$ whose weak derivative lies in $L_{\infty}(\mathbb{R})$ is indeed a Lipschitz-continuous function  \cite[Theorem 1.36]{weaver1999lipschitz}. In other words, we have that  
\begin{equation}\label{Eq:LipNative}
{\rm Lip}(\mathbb{R})= \{f\in \mathcal{D}'(\mathbb{R}): {\rm D}\{f\} \in L_{\infty}(\mathbb{R})\},
\end{equation}
which allows us to view ${\rm Lip}(\mathbb{R})$ as the native Banach space associated to the pair $(L_{\infty}(\mathbb{R}), {\rm D})$ in the sense of  \cite{unser2019native}. 

\subsection{Second-Order Total-Variation}
To conclude this section, we introduce the space ${\rm BV}^{(2)}(\mathbb{R})$ of functions with finite second-order total-variation,   defined as 
\begin{align}
    {\rm TV}^{(2)}(f) &= \|{\rm D}^2\{f\}\|_{\mathcal{M}}= \sup_{\substack{\varphi\in \mathcal{D}(\mathbb{R}) \\ \|\varphi\|_{L_\infty}=1}} \langle {\rm D}^2 f, \varphi \rangle \\&=  \sup_{\substack{\varphi\in \mathcal{D}(\mathbb{R})\\ \|\varphi\|_{L_\infty}=1}} \int_{\mathbb{R}} f(x) \varphi''(x) {\rm d}x.    \label{Eq:TV2Def}
\end{align}

Analogous to the famous total-variation regularization of Rudin-Osher-Fatemi \cite{rudin1992nonlinear}, which promotes piecewise-constant functions and causes the notorious staircase effect, the second-order total variation favors CPWL functions. In dimension $d=1$, this coincides with the known class of nonuniform linear splines which has been extensively studied from an approximation-theoretical point of view \cite{de1978practical,unser1999splines}. Motivated by this, the ${\rm TV}^{(2)}$ regularization has been exploited to learn activation functions of deep neural networks \cite{unser2019Deepspline,Bohra2020DeepSpline}. In a similar vein, the identification of the  sparsest CPWL solutions of   ${\rm TV}^{(2)}$-regularized problems has been thoroughly  studied in \cite{debarre2020sparsest}. 

\section{Lipschitz-Aware Formulations for Supervised Learning}\label{sec:Proposals}
We now introduce our formulations for supervised learning that are based on controlling the Lipschitz constant of the learned mapping. Let us first mention that the Lipschitz constant can be indirectly controlled using a ${\rm TV}^{(2)}$-type regularizer. Indeed, the two seminorms are connected, as demonstrated in    Theorem \ref{thm:LipTV}. 

\begin{theorem}\label{thm:LipTV}
Any function  with second-order bounded-variation is Lipschitz continuous. Moreover, for any $f\in {\rm BV}^{(2)}(\mathbb{R})$, we have the upper-bound
\begin{equation}\label{Eq:LipTV}
    { L}(f) \leq {\rm TV}^2(f)+ \ell(f)
\end{equation}
for the Lipschitz constant of $f$, where 
\begin{equation}
\ell(f)= \inf_{x_1\neq x_2} \frac{ \left|f(x_1)-f(x_2)\right|}{ \left|x_1-x_2\right|} \geq 0.  
\end{equation}
Finally, \eqref{Eq:LipTV} is saturated if and only if $f$ is  monotone and convex/concave. 
\end{theorem}
The proof of Theorem \ref{thm:LipTV} is given in Appendix \ref{App:LipTV}. A weaker version of this theorem is proven in \cite{Aziznejad2020DNN}, where $\ell(f)$ is replaced with $|f(1)-f(0)|$, which is clearly an upper-bound. The importance of the updated bound is that it is sharp in the sense that it is an equality for monotone and convex/concave functions. 

 A weaker version of \eqref{Eq:LipTV} motivated the authors of \cite{Aziznejad2020DNN} to provide a global bound for the Lipschitz constant of deep neural networks and to regularize it during training. Although this is an interesting approach to control the Lipschitz constant of the learned mapping, the obtained guarantee is too conservative. This is due to the fact that, as soon as $f$ has some oscillations, the difference between the two sides of \eqref{Eq:LipTV}  dramatically increases and the bound becomes  loose. Here, by contrast, we shall ensure  the global stability of the learned mapping by directly  controlling the Lipschitz constant itself. 

\subsection{Lipschitz  Regularization}
We first consider the Lipschitz constant as a regularizer and study the minimization problem 
 \begin{equation}\label{Eq:LipMin}
    \mathcal{V}_{\rm Lip}=\argmin_{f\in {\rm Lip}(\mathbb{R})} \left( \sum_{m=1}^M { E}(f(x_m),y_m) + \lambda L(f)\right),
 \end{equation}
 where ${ E}:\mathbb{R}\times \mathbb{R}\rightarrow \mathbb{R}$ is a strictly convex and coercive function and where $\lambda>0$ is the regularization parameter.  We also assume, without   loss of generality, that the data points $x_m$ are sorted in the increasing order $x_1< x_2 <\cdots< x_M$. In Theorem \ref{Thm:RepLipMin}, we state our main theoretical contributions regarding the minimization problem \eqref{Eq:LipMin}. 
 
\begin{theorem}\label{Thm:RepLipMin}
Regarding the minimization problem \eqref{Eq:LipMin}, the following statements hold.
\begin{enumerate}
\item \label{ItemRepLipMin:Exist} The solution set $\mathcal{V}_{\rm Lip}$ is a nonempty, convex and weak*-compact subset of ${\rm Lip}(\mathbb{R})$. 

\item  \label{ItemRepLipMin:z} There exists a unique vector ${\bf z}=(z_m)\in\mathbb{R}^M$ such that 
 \begin{equation}\label{Eq:LipMinConst}
    \mathcal{V}_{\rm Lip}=\argmin_{f\in {\rm Lip}(\mathbb{R})}  L(f), \quad \text{s.t.} \quad f(x_m)=z_m, \quad \forall m.
 \end{equation}

\item \label{ItemRepLipMin:LipOpt} The optimal Lipschitz constant has the closed-form expression 
\begin{equation}\label{Eq:LipOpt}
L_{\min}= \max_{2\leq m\leq M} \left| \frac{z_m-z_{m-1}}{x_m-x_{m-1}}\right|.
\end{equation} 
Consequently, any $L_{\min}$-Lipschitz function $f$ that satisfies $f(x_m)=z_m, m=1,\ldots,M$ is a solution of \eqref{Eq:LipMin}.

\item  \label{ItemRepLipMin:Graph} Let $\mathcal{E}\subseteq \mathbb{R}^2$ be the union of the graphs of all solutions of \eqref{Eq:LipMin}, defined as
\begin{equation}\label{Eq:UnionGraph}
\mathcal{E}= \left\{(x,y)\in\mathbb{R}^2: \exists f\in \mathcal{V}_{\rm Lip},  y=f(x)\right\}.
\end{equation}
Let us also define the   right and left planar cones $\mathcal{R},\mathcal{L}\subseteq \mathbb{R}^2$ as 
\begin{equation}\label{Eq:LeftRightCone}
\mathcal{R}= \left\{\alpha_1 (1,L_{\min}) + \alpha_2 (1,-L_{\min}): \alpha_1,\alpha_2\geq 0\right\},   
\end{equation}
and $\mathcal{L}=-\mathcal{R}$. With the convention  that $\mathcal{R}_0 = \mathcal{L}_{M+1}= \mathbb{R}^2$, we have that 
\begin{equation}
\mathcal{E}= \bigcup_{m=1}^{M+1} \left( \mathcal{R}_{m-1}\cap \mathcal{L}_m \right),
\end{equation}
where the $\mathcal{R}_m$ and $\mathcal{L}_{m}$ are   shifted versions of $\mathcal{R}$ and $\mathcal{L}$, with    
\begin{equation}\label{Eq:RmLm}
\mathcal{R}_m = (x_m,z_m) + \mathcal{R}, \quad  \mathcal{L}_m = (x_m,z_m) + \mathcal{L}, \forall m.
\end{equation}

\item \label{ItemRepLipMin:TV2} Any solution of the constrained minimization problem
 \begin{equation}\label{Eq:TV2MinConst}
    \min_{f\in {\rm BV}^{(2)}(\mathbb{R})}  {\rm TV}^{(2)}(f), \quad \text{s.t.} \quad f(x_m)=z_m, 1\leq m\leq M
 \end{equation}
is included in $\mathcal{V}_{\rm Lip}$. In particular, the solution set of \eqref{Eq:LipMin} always includes a   continuous and piecewise-linear function.
\end{enumerate}
\end{theorem}
The proof of Theorem \ref{Thm:RepLipMin} is given in Appendix \ref{App:RepLipMin}. Items \ref{ItemRepLipMin:Exist} and \ref{ItemRepLipMin:z} are classical results that hold for a general class of variational problems (see  \cite{unser2020unifying} for a generic result). Their practical implication is Item \ref{ItemRepLipMin:LipOpt}, which provides a way to identify solutions of \eqref{Eq:LipMin}. The solution set $\mathcal{V}_{\rm Lip}$ is further explored in Item \ref{ItemRepLipMin:Graph}, where a geometrical insight is given (see Figure \ref{Fig:Graph}).
\begin{figure}[t]
\centering
\includegraphics[width=\linewidth]{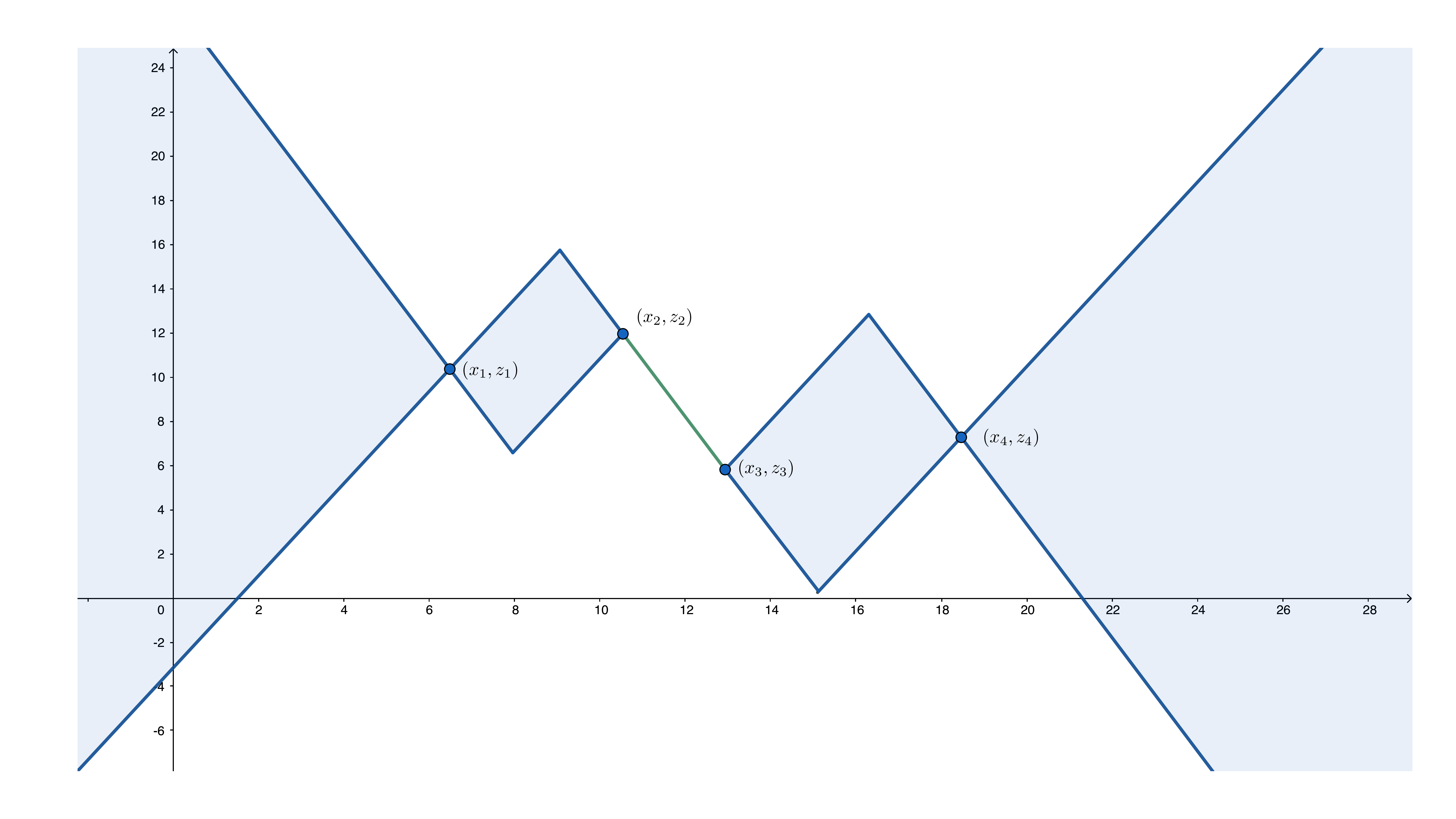}
\caption{The union of the graphs of all solutions in a simple example with four data points. Note that all solutions must directly connect $(x_2,z_2)$ to $(x_3,z_3)$, since the slope of this segment is $L_{\min}$ whose formula is given in \eqref{Eq:LipOpt}.}
\label{Fig:Graph}
\end{figure}
Finally, the   result that has the greatest practical relevance is stated in Item \ref{ItemRepLipMin:TV2} which creates an interesting link with ${\rm TV}^{(2)}$ minimization problems and hence guarantees the existence of CPWL solutions.
\subsection{Lipschitz Constraint}
While the first formulation is interesting on its own right and results in learning CPWL mappings with tunable Lipschitz constants, it does not necessarily yield a sparse (and, hence, interpretable) solution. In fact, the  learned mapping can have undesirable oscillations  as illustrated in Figure \ref{fig:limitation_lip}. This observation motivates us to propose a second formulation that combines ${\rm TV}^{(2)}$ regularization with a constraint over  the  Lipschitz constant, as expressed by 
\begin{align}
    \mathcal{V}_{\rm hyb}=\argmin_{f\in {\rm BV}^{(2)}(\mathbb{R})} &\left(\sum_{m=1}^M {\rm E}(f(x_m),y_m) + \lambda {\rm TV}^{(2)}(f)\right),   \nonumber\\&\text{s.t.} \quad L(f)\leq \bar{L}.\label{Eq:Hybrid}
\end{align}

The quantity $\bar{L}$ is the maximal  value allowed for the Lipschitz constant of the learned mapping. In this way, the stability is directly controlled by the user, while  the regularization term removes   undesired oscillations (tunable with $\lambda>0$). The solution set  $\mathcal{V}_{\rm hyb}$ is characterized in Theorem \ref{Thm:Hybrid}, from which we also deduce the existence of CPWL solutions. 
 \begin{theorem}\label{Thm:Hybrid}
The solution set $\mathcal{V}_{\rm hyb}$ of Problem \eqref{Eq:Hybrid} is a nonempty, convex, and weak*-compact subset of ${\rm BV}^{(2)}(\mathbb{R})$ whose extreme points are linear splines with at most $(M-1)$ linear regions. Moreover, there exists a unique vector ${\bf z}=(z_m)$ such that 
\begin{equation}\label{Eq:TV2MinConst2}
\mathcal{V}_{\rm hyb}=  \argmin_{f\in {\rm BV}^{(2)}(\mathbb{R})}  {\rm TV}^{(2)}(f), \quad \text{s.t.} \quad f(x_m)=z_m, 1\leq m\leq M.
 \end{equation} 
Finally, the optimal ${\rm TV}^{(2)}$ cost has the closed-form expression 
\begin{equation}\label{Eq:TV2Optimal}
{\rm TV}_{\min}= \sum_{m=2}^{M-1} \left| \frac{z_m-z_{m-1}}{x_m-x_{m-1}}-\frac{z_m-z_{m+1}}{x_m-x_{m+1}}\right|.
\end{equation}
 \end{theorem}
The proof is given in Appendix \ref{App:Hybrid}. The proof involves the weak*-closedness of the constraint box $L(f)\leq \overline{L}$ which is essential to prove existence. Once the existence of a minimizer is guaranteed, we can invoke the results of Debarre {\it et al.} in \cite{debarre2020sparsest} for ${\rm TV}^{(2)}$ minimization to deduce the remaining parts. We also remark that the Lipschitz constraint only affects the vector ${\bf z}$ in \eqref{Eq:TV2MinConst2}, which forces its entries to satisfy the inequalities
\begin{equation}
\left| \frac{z_m-z_{m-1}}{x_m-x_{m-1}}\right| \leq \overline{L}, \quad m=2,\ldots, M.
\end{equation}

\subsection{Connection to Neural Networks}
In this part, we show that our second formulation \eqref{Eq:Hybrid} is equivalent to training a two-layer neural network with weight decay. Let us recall that a univariate ReLU network with two layers and skip connections is a mapping $f_{\boldsymbol{\theta}}:\mathbb{R}\rightarrow\mathbb{R}$ of the form 
\begin{equation}\label{Eq:TwoLayerNN}
  f_{\boldsymbol{\theta}}(x)= c_0+ c_1 x + \sum_{k=1}^K v_k {\rm ReLU}(w_k x - b_k),   
\end{equation}
where $c_1\in\mathbb{R}$ is the weight of the skip connection, $K\in\mathbb{N}$ is the width of the network, $v_k,w_k\in\mathbb{R},k=1,\ldots,K$ are the linear weights and $b_k\in\mathbb{R}$, $k=1,\ldots,K$ and $c_0\in \mathbb{R}$ are the bias terms of the first and second layers, respectively. These parameters are concatenated in a single vector $\boldsymbol{\theta}=(K,\mathbf{v},\mathbf{w},\mathbf{b},\mathbf{c})$, and we denote by $\boldsymbol{\Theta}$ the set of all possible parameter vectors $\boldsymbol{\theta}$. Thus, the training problem with Lipschitz constraint and weight decay is formulated as
\begin{align}
    \mathcal{V}_{NN}= \argmin_{\boldsymbol{\theta}\in\boldsymbol{\Theta}}&\left( \sum_{m=1}^M {\rm E}(f_{\boldsymbol{\theta}}(x_m),y_m) +  \lambda {\rm R}(\boldsymbol{\theta})  \right), \nonumber \\ & \text{s.t.} \quad L(f_{\boldsymbol{\theta}})\leq \bar{L}, \label{Eq:NNlearning}
\end{align}
where ${\rm R}(\boldsymbol{\theta}) = \sum_{k=1}^K \left( \frac{|v_k|^2 + |w_k|^2}{2}\right)$ is the regularization term corresponding to weight decay. In Proposition~\ref{prop:equivalence_training_NN}, we show the equivalence between this training problem and our Lipschitz-constrained formulation \eqref{Eq:Hybrid}.
\begin{proposition}
\label{prop:equivalence_training_NN}
For any solution $\boldsymbol{\theta}^*$ of \eqref{Eq:NNlearning}, $f_{\boldsymbol{\theta}^*}$ is a CPWL solution of \eqref{Eq:Hybrid}. Moreover, any CPWL solution of \eqref{Eq:Hybrid} can be expressed as a two-layer ReLU network $f_{\boldsymbol{\theta}^*}$ with skip connections whose parameter vector is optimal in the sense of \eqref{Eq:NNlearning}, \ie $\boldsymbol{\theta}^* \in \mathcal{V}_{NN}$. 
\end{proposition}
The proof of Proposition \ref{prop:equivalence_training_NN} can be found in Appendix \ref{App:trainingNN}. The latter is a direct extension of the results of \cite{savarese2019how, parhi2020role}, where this equivalence is proved in the absence of a Lipschitz constraint. The interesting outcome of Proposition~\ref{prop:equivalence_training_NN} is that it provides a functional framework to study the training of Lipschitz-aware neural networks, which can be used to analyze their behavior from a theoretical perspective. In particular, our description of the solution set of Problem~\eqref{Eq:Hybrid} (Theorem~\ref{Thm:Hybrid}) and of its sparsest solutions (Theorem~\ref{Thm:DebarreOptimal}) provide interesting insights, for example on how to choose the number of neurons $K$.

 \section{Finding the Sparsest CPWL Solution}\label{sec:Algo}
 Using the theoretical results of Section \ref{sec:Proposals}, we propose an algorithm to find the sparsest CPWL solution of Problems \eqref{Eq:LipMin} and \eqref{Eq:Hybrid}. To that end, we first compute the vector $\bf z$ of the value of the optimal function at the data points $x_1,\ldots,x_m$. Using this  vector, we then deploy the sparsification algorithm of \cite{debarre2020sparsest}, whose use in the present method is motivated by  the following theorem.
 \begin{theorem}\label{Thm:DebarreOptimal}
Let $(x_m,z_m)\in\mathbb{R}^2,m=1,\ldots,M$ be a collection of ordered data points with $x_1 < \cdots < x_M$. Then, the output $f_{\mathrm{sparse}}$ of the sparsification algorithm of Debarre {\it et al.} in \cite{debarre2020sparsest} is the sparsest linear-spline interpolator of the data points. In other words, $f_{\mathrm{sparse}}$ is the CPWL interpolator with the fewest number of linear regions. 
\end{theorem}
The proof is given in Appendix \ref{App:DebarreOptimal}. Theorem~\ref{Thm:DebarreOptimal} is a strong enhancement of \cite[Theorem 4]{debarre2020sparsest} where it is merely established that  $f_{\mathrm{sparse}}$ is the sparsest CPWL solution of \eqref{Eq:TV2MinConst}. In Theorem \ref{Thm:DebarreOptimal}, we prove that $f_{\mathrm{sparse}}$ is in fact the sparsest of \emph{all} CPWL interpolants of the data points $(x_m,z_m)$, without restricting the search to the solutions of \eqref{Eq:TV2MinConst}. This is a remarkable result in its own right, as it gives a nontrivial answer to the seemingly simple question: how to interpolate data points with the minimum number of lines? Here, we invoke Theorem~\ref{Thm:DebarreOptimal} to deduce that, with the vector $\V z$ defined in Item 2 of Theorem~\ref{Thm:RepLipMin}, $f_{\mathrm{sparse}}$ is the sparsest CPWL solution of \eqref{Eq:LipMinConst}. Similarly, with the vector $\V z$ defined in Theorem~\ref{Thm:Hybrid}, $f_{\mathrm{sparse}}$ is the sparsest CPWL solution of \eqref{Eq:Hybrid}. 

In the remaining part of this section, we detail our computation of the vectors ${\bf z}$ defined in Theorems~\ref{Thm:RepLipMin} and \ref{Thm:Hybrid}. Let us define the empirical loss function $F:\mathbb{R}^M \rightarrow \mathbb{R}_{\geq 0}$ as 
\begin{equation}\label{Eq:DataFid}
F({\bf z}) = \sum_{m=1}^M E(z_m,y_m). 
\end{equation}
For simplicity, we assume that $F$ is differentiable; the prototypical example is the quadratic loss $F({\bf z}) = \frac{1}{2} \sum_{m=1}^M (z_m-y_m)^2$. Following this notation and using \eqref{Eq:LipOpt}, the vector ${\bf z}$ in Problem \eqref{Eq:LipMinConst} is solution to the minimization problem 
\begin{equation}\label{Eq:RedLipMin}
\min_{{\bf z}\in \mathbb{R}^M} \left( F({\bf z}) + \lambda \|{\bf L}_{\rm inf} {\bf z}\|_{\infty}\right),
\end{equation}
where the matrix ${\bf L}_{\rm inf} \in \mathbb{R}^{(M-1)\times M}$ is given by 
\begin{equation}\label{Eq:Linf}
[{\bf L}_{\rm inf}]_{m,n} = \begin{cases} -v_{m+1}, & n=m 
\\ v_{m+1}, & n=m+1 
\\ 0, & \text{otherwise}
\end{cases}
\end{equation}
 where  $v_m = (x_m -x_{m-1})^{-1}, m=2,\ldots,M$. To solve \eqref{Eq:RedLipMin}, we use the well-known alternating-direction method of multipliers
(ADMM) \cite{boyd2011distributed} by defining the augmented Lagrangian as 
\begin{equation}
J({\bf z},{\bf u},{\bf w}) = F({\bf z}) + \lambda \|{\bf u}\|_{\infty} + \frac{\rho}{2} \|{\bf L}_{\inf}{\bf z} - {\bf u}\|_2^2 + {\bf w}^T({\bf L}_{\inf}{\bf z}- {\bf u}),
\end{equation}
where $\rho>0$ is a tunable parameter. The principle of ADMM is to sequentially update the unknown variables ${\bf z}\in\mathbb{R}^M$ and ${\bf u},{\bf w}\in \mathbb{R}^{M-1}$. Precisely, its $k$th iteration is given explicitly by
\begin{align}
{\bf z}^{(k+1)}&= \argmin_{{\bf z}\in\mathbb{R}^M} J( {\bf z}, {\bf u}^{(k)}, {\bf w}^{(k)}), \label{Eq:UpdateZ1} \\
{\bf u}^{(k+1)}&= \argmin_{{\bf u}\in\mathbb{R}^{M-1}} J( {\bf z}^{(k+1)}, {\bf u}, {\bf w}^{(k)}),
 \label{Eq:UpdateU1}\\
 {\bf w}^{(k+1)}&= {\bf w}^{(k)} + \rho \left( {\bf L}_{\inf} {\bf z}^{(k+1)} - {\bf u}^{(k+1)} \right). \label{Eq:UpdateW1}
\end{align}
The benefit of these sequential updates  is that Problem  \eqref{Eq:UpdateZ1}    has a differentiable cost  and hence, can be efficiently solved using gradient-based methods. (In the case of the quadratic loss $E(z,y)=\frac{1}{2}(z-y)^2$, one can even obtain a closed-form solution.) Unfortunately, the cost in \eqref{Eq:UpdateU1} is not differentiable. However, one can rewrite  the   augmented Lagrangian as 
\begin{equation}
J( {\bf z}_k, {\bf u}, {\bf w}_k) = \frac{\rho}{2} \left\|{\bf u}- {\bf L}_{\inf}{\bf z}_k - \frac{1}{\rho} {\bf w}_k\right\|_2^2 + \lambda \|{\bf u}\|_{\infty} + \text{Cnst.},
\end{equation}
where the constant term accounts for  all terms that do not depend on ${\bf u}$.  Then, by defining the vector ${\bf v}_k= \left({\bf L}_{\inf}{\bf z}_k+\frac{1}{\rho}{\bf w}_k\right)$, we rewrite \eqref{Eq:UpdateU1}  as 
\begin{align}
{\bf u}^{(k+1)}&= \argmin_{{\bf u}\in\mathbb{R}^{M-1}} \left(\frac{1}{2}\|{\bf u}- {\bf v}_k\|_2^2 + \frac{\lambda}{\rho} \|{\bf u}\|_{\infty}\right) \nonumber \\ &= {\rm prox}_{\frac{\lambda}{\rho}\|\cdot\|_{\infty}} ( {\bf v}_k),\label{Eq:ProxLinf}
\end{align}
by definition of the proximal operator. The proximal operator of the $\ell_\infty$-norm has computationally cheap implementations    (see, for example, \cite[Section 6.5.2]{parikh2014proximal}), which can be used  to update ${\bf u}$ via   \eqref{Eq:ProxLinf}.

Similarly and using \eqref{Eq:TV2Optimal}, we formulate the search for the vector ${\bf z}$ associated to the Problem \eqref{Eq:Hybrid}   as 
\begin{equation}\label{Eq:RedHybrid}
\min_{{\bf z}\in \mathbb{R}^M}  \left(F({\bf z}) + \lambda \|{\bf L}_{1} {\bf z}\|_{1}+ i_{\|{\bf L}_{\rm inf}{\bf z}\|_{\infty} \leq \overline{L}}\right),
\end{equation}
where $i_E$ denotes the indicator function of the set $E$ and ${\bf L}_1\in\mathbb{R}^{(M-2)\times M}$ with 
\begin{equation}\label{Eq:L1}
[{\bf L}_{1}]_{m,n} = \begin{cases} -v_{m+1}, & n=m 
\\ (v_{m+1}+v_{m+2}), & n=m+1 
\\ -v_{m+2}, & n=m+2,
\\ 0, & \text{otherwise}
\end{cases}
\end{equation}
for all $m=1,\ldots,M-2$ and $n=1,\ldots,M$. In this case, the augmented Lagrangian takes the form 
\begin{align} 
&J({\bf z}, {\bf u}_1,{\bf u}_{\inf} , {\bf w}_1,{\bf w}_{\inf})=F({\bf z}) \nonumber  \\&+ \frac{\rho_1}{2} \|{\bf L}_1 {\bf z}- {\bf u}_1\|_2^2 + {\bf w}_1^T({\bf L}_1{\bf z}-{\bf u}_1) + \|{\bf u}_1\|_1 \nonumber\\&+ \frac{\rho_{\inf}}{2}\|{\bf L}_{\inf}{\bf z}- {\bf u}_{\inf} \|_2^2 +{\bf w}_{\inf}^T({\bf L}_{\inf}{\bf z}-{\bf u}_{\inf})+ i_{\|{\bf u}_{\inf}\|_{\infty} \leq \overline{L}}. 
\end{align}
At the $k$th iteration, we then solve sequentially the following optimization problems
\begin{align}
{\bf z}^{(k+1)}&= \argmin_{{\bf z}\in\mathbb{R}^M} J( {\bf z}, {\bf u}_1^{(k)},{\bf u}_{\inf}^{(k)}, {\bf w}_1^{(k)},{\bf w}_{\inf}^{(k)}), \label{Eq:UpdateZ2} \\
{\bf u}_1^{(k+1)}&= \argmin_{{\bf u}_1\in\mathbb{R}^{M-2}} J( {\bf z}^{(k+1)}, {\bf u}_1,{\bf u}_{\inf}^{(k)}, {\bf w}_1^{(k)},{\bf w}_{\inf}^{(k)}),
 \label{Eq:UpdateU21}\\
{\bf u}_{\inf}^{(k+1)}&= \argmin_{{\bf u}_{\inf}\in\mathbb{R}^{M-1}} J( {\bf z}^{(k+1)}, {\bf u}_1^{(k+1)},{\bf u}_{\inf}, {\bf w}_1^{(k)},{\bf w}_{\inf}^{(k)}),
 \label{Eq:UpdateU22}\\
 {\bf w}_1^{(k+1)}&= {\bf w}_1^{(k)} + \rho_1 \left({\bf L}_1{\bf z}^{(k+1)}-{\bf u}_1^{(k+1)}\right) ,
  \label{Eq:UpdateW21}\\
 {\bf w}_{\inf}^{(k+1)}&= {\bf w}_{\inf}^{(k)} + \rho_{\inf} \left({\bf L}_{\inf}{\bf z}^{(k+1)}-{\bf u}_{\inf}^{(k+1)}\right) . \label{Eq:UpdateW22} 
\end{align}
The cost function  of Problem  \eqref{Eq:UpdateZ2} is  differentiable and so, we can solve it using gradient-based methods. For Problem \eqref{Eq:UpdateU21}, we invoke the proximal operator of the $\ell_1$-norm that is  known to be soft-thresholding \cite[Section 6.5.2.]{parikh2014proximal}. Finally and for \eqref{Eq:UpdateU22}, the proximal operator of the indicator function $i_{\|{\cdot}\|_\infty\leq \overline{L}}$ is the projection over the $\ell_{\infty}$ ball which has the simple separable expression 
\begin{equation}
[{\rm prox}_{i_{\|{\cdot}\|_\infty\leq \overline{L}}} ({\bf v})]_n =\begin{cases}\overline{L}, & v_n> \overline{L} \\ v_n, & |v_n|\leq \overline{L}\\ -\overline{L},& v_n< -\overline{L}. \end{cases} 
\end{equation}

\section{Numerical Examples and Discussions}\label{sec:Numerical}
 
 \subsection{Experimental Setup}
 \label{sec:setup}
 In all our experiments, we consider the standard quadratic loss $E(y, z) = \frac12 (y - z)^2$. We draw the data-point locations $x_m$ randomly in the interval $[0, 1]$. 
The   values $y_m$ are then generated as $y_m = f_0(x_m) + n_m$, where $f_0$  is some known CPWL function (gold standard) and $n_m$ is drawn i.i.d.\ from a zero-mean normal distribution with   variance $\sigma^2$.

\subsection{Example of Lipschitz Regularization}

In this first experiment, we illustrate our first formulation~\eqref{Eq:LipMin}. We take $M=50$ data points, a CPWL ground-truth $f_0$ with $6$ linear regions, and a noise level $\sigma = 0.02$.

\begin{figure}
    \centering
 \subfloat[Reconstructions for different values of $\lambda$. Number of linear regions: 10 for $\lambda = 0.029$ versus 37 for $\lambda = +\infty$.]{\includegraphics[width=\linewidth,valign=t]{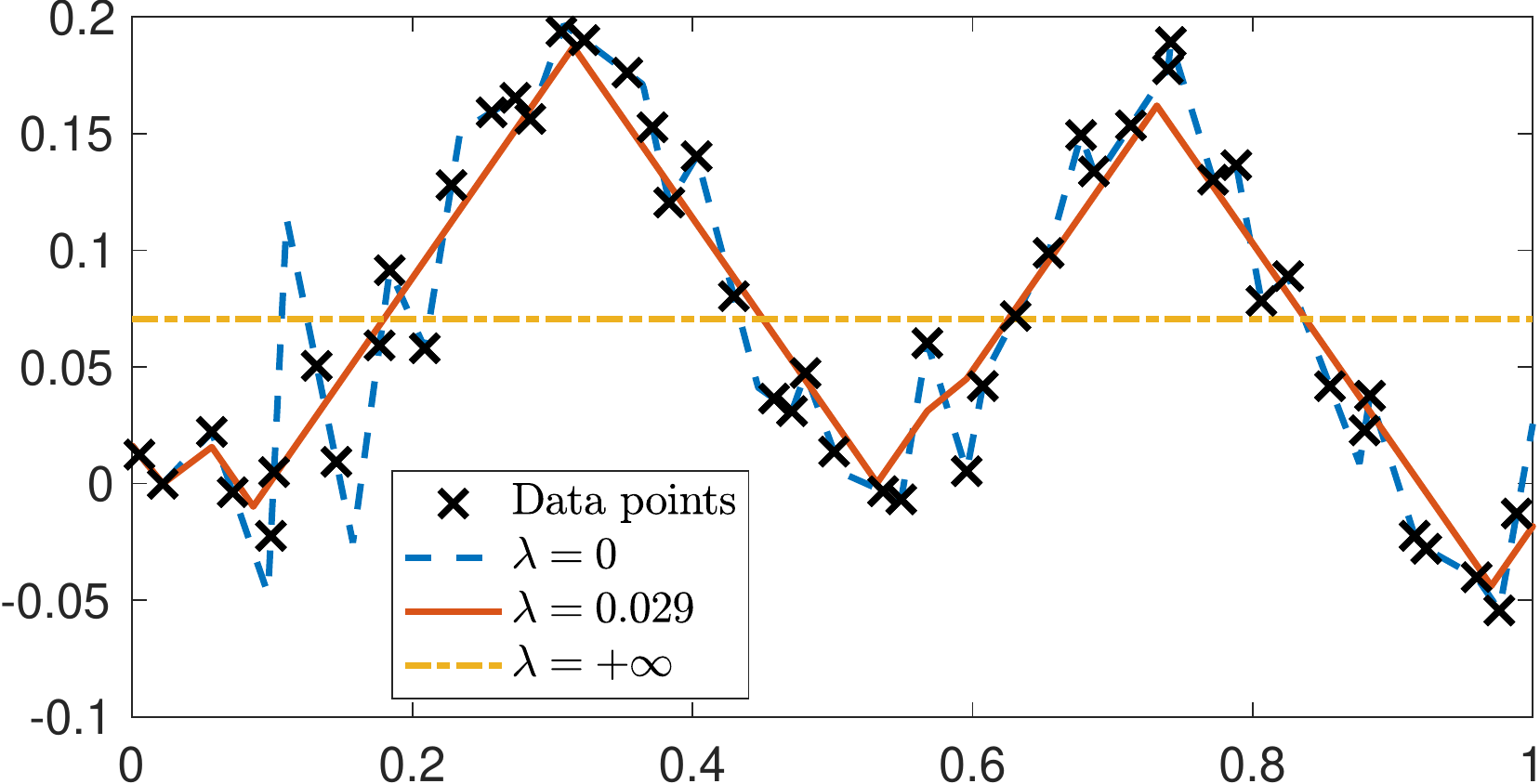} \label{fig:lip_reconstruction}} \\
    \subfloat[Evolution of the training error and the Lipschitz regularity with respect to $\lambda$. The diamond corresponds to $\lambda = 0.029$ (shown in Figure~\ref{fig:lip_reconstruction}).]{\includegraphics[width=\linewidth,valign=t]{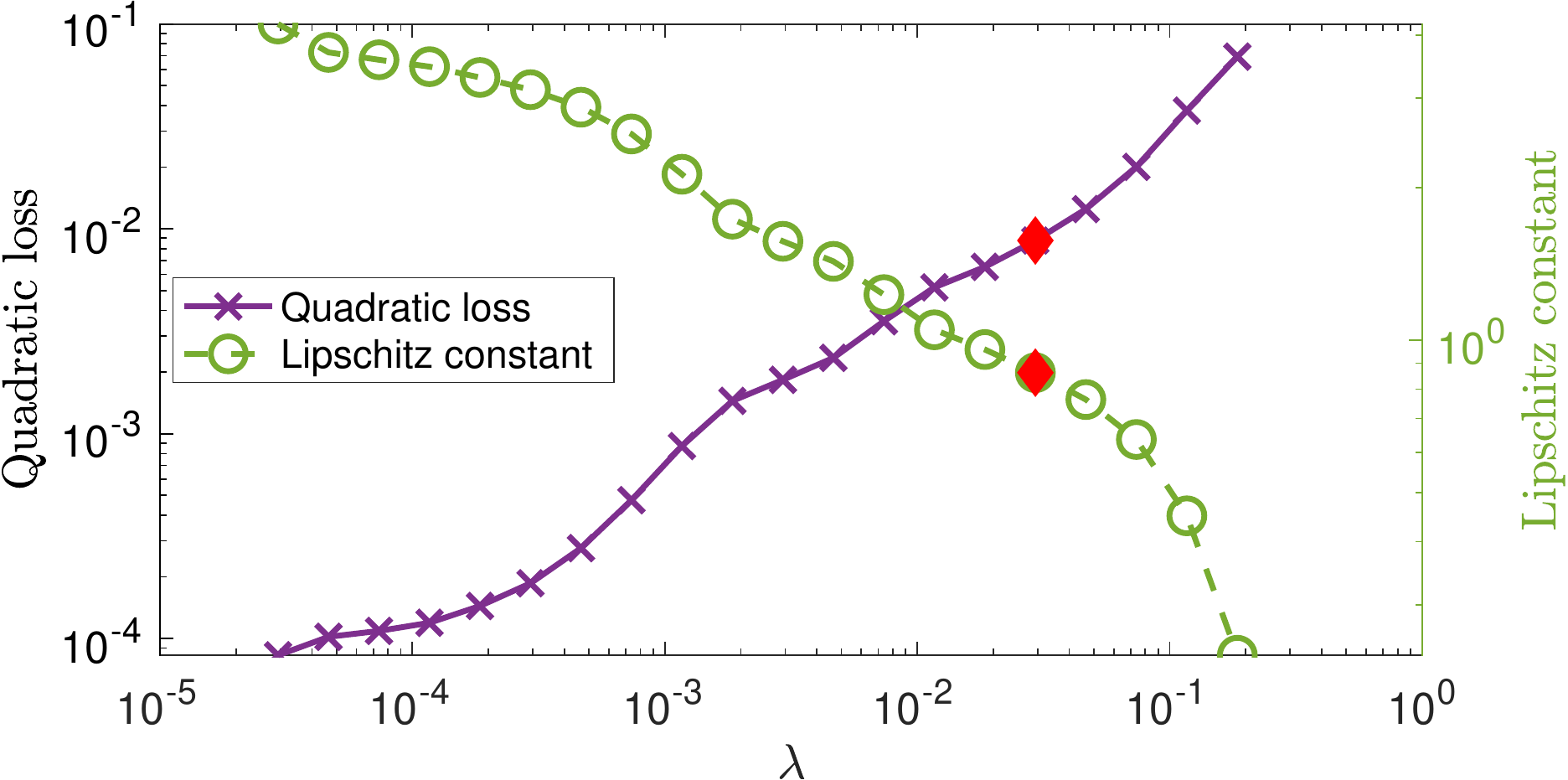}\label{fig:lip_costs}}
    \caption{Example of our first formulation~\eqref{Eq:LipMin} for $M=50$ data points.}
    \label{fig:lip}
\end{figure}

The results are shown in Figure~\ref{fig:lip}. In Figure~\ref{fig:lip_reconstruction}, we show the reconstructions for extreme values of $\lambda$. On one hand, $\lambda \to 0$ corresponds to the exact interpolation Problem~\eqref{Eq:LipMinConst}. On the other hand, $\lambda = +\infty$ corresponds to constant regression. Obviously, neither is very satisfactory: interpolation leads to overfitting (the reconstruction has 37 linear regions), and the constant regression to underfitting. We show an example of a more satisfactory reconstruction for $\lambda = 0.029$ (10 linear regions), which is visually acceptable. In Figure \ref{fig:lip_costs}, we show the evolution of the quadratic loss $\frac12 \sum_{m=1}^M (f^\ast(x_m) - y_m)^2$ and the Lipschitz constant $L(f^\ast)$, for various values of $\lambda$. With the aid of such curves, the user can choose what is considered acceptable for either of these costs and select a suitable value of $\lambda$.
 
 \subsection{Limitations of Lipschitz-Only Regularization}
 
 Despite its interesting theoretical properties, Problem~\eqref{Eq:LipMin} does not always yield satisfactory reconstructions. This is because it does not enforce a sparse reconstruction in the problem formulation, despite the fact that our algorithm reconstructs (one of) the sparsest elements of $\mathcal{V}_{\rm lip}$. This leads to learned mappings with too many linear regions and, consequently, poor interpretability.
 
 One such example is shown in Figure~\ref{fig:limitation_lip}, where we consider the shifted ReLU function $f_0(\cdot) = (\cdot-\frac{1}{2})_+$ as the ground-truth mapping. We also fix  the standard deviation of the noise to $\sigma = 0.02$. Figure~\ref{fig:relu_002} shows a reconstruction that solves Problem~\eqref{Eq:LipMin} with the regularization parameter $\lambda=0.02$. Although  the  reconstruction is satisfactory in the active section ($x > 1/2$), it has many linear regions in the flat section ($x<1/2$) that are not present in $f_0$. This is due to the fact that the active section forces the Lipschitz constant of the reconstruction to be around 1,  while oscillations with a slope  smaller  than 1 in the flat section are not penalized by the regularization. This problem  clearly cannot be fixed by a simple increase in the regularization parameter: with $\lambda = 0.2$ (Figure~\ref{fig:relu_02}), not only there are still too many linear regions in the flat section (the reconstruction has 9 linear regions in total), but also the active section is poorly reconstructed because the Lipschitz constant is  penalized too heavily by the regularization.
 
\begin{figure*}
    \centering
 \subfloat[][Lipschitz regularization. \\ Number of linear regions: 13.]{\includegraphics[width=0.33\linewidth]{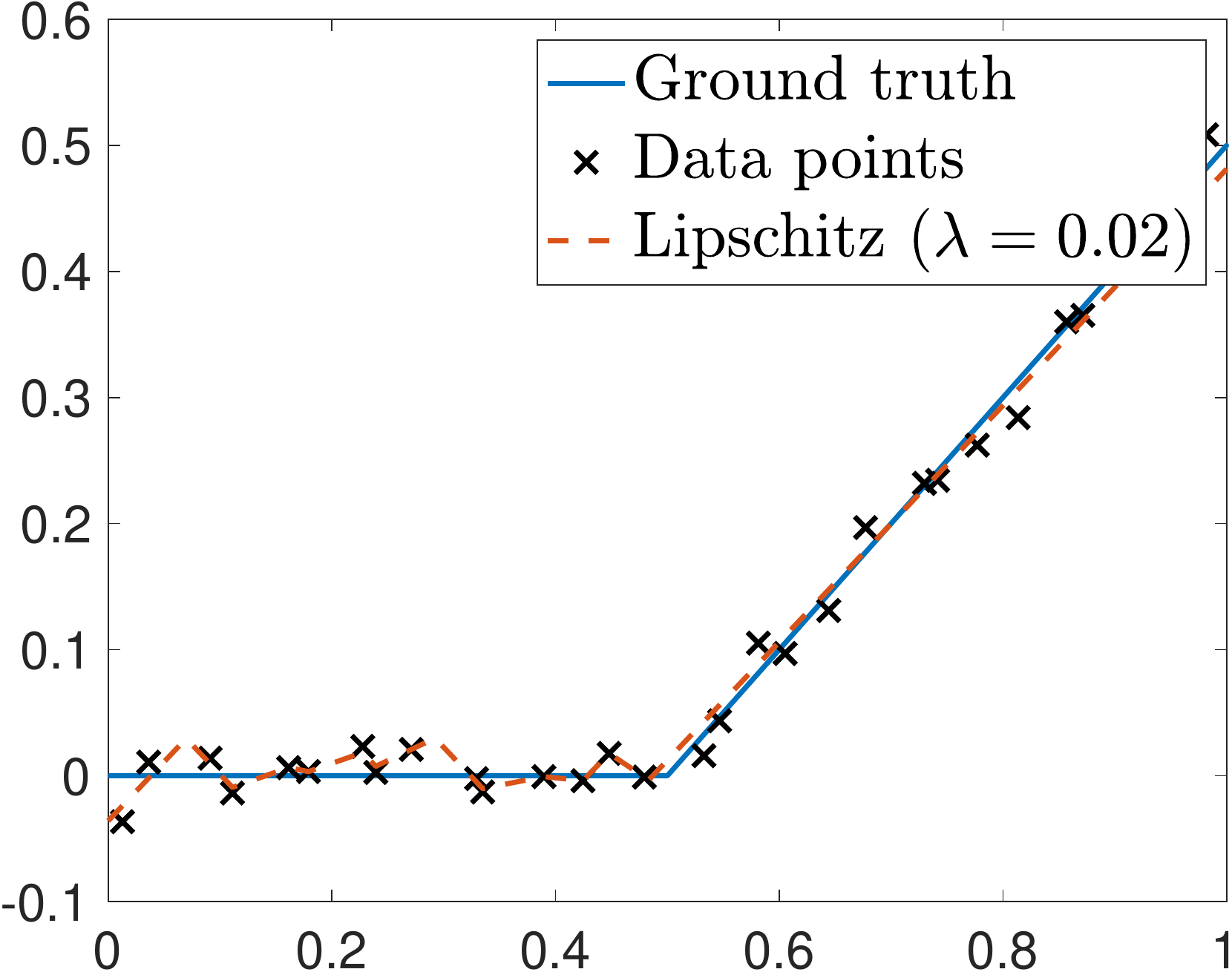} \label{fig:relu_002}}
    \subfloat[][Lipschitz regularization. \\ Number of linear regions: 9.]{\includegraphics[width=0.33\linewidth]{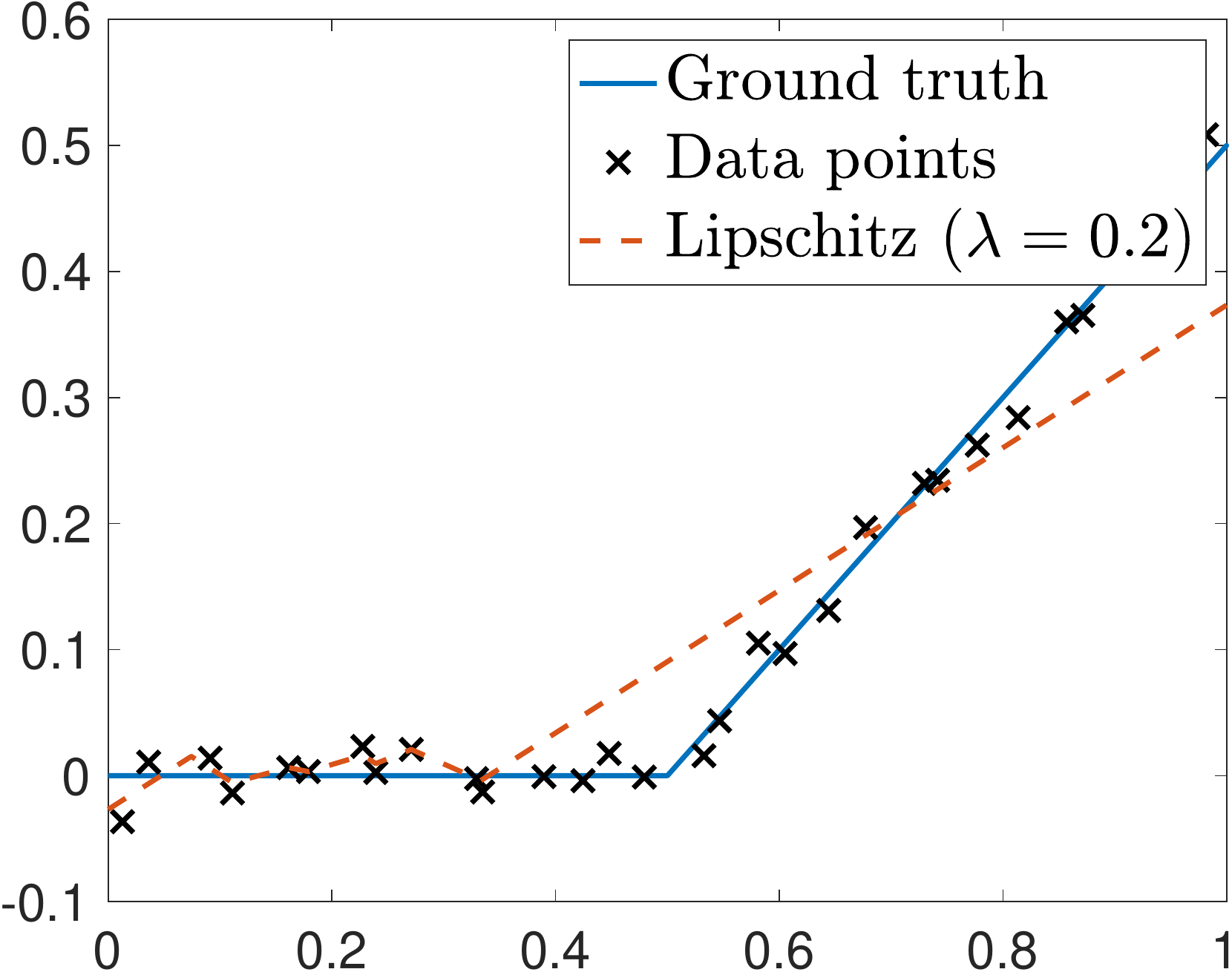}\label{fig:relu_02}} 
     \subfloat[][${\rm TV}^{(2)}$ regularization. \\ Number of linear regions: 2.]{\includegraphics[width=0.33\linewidth]{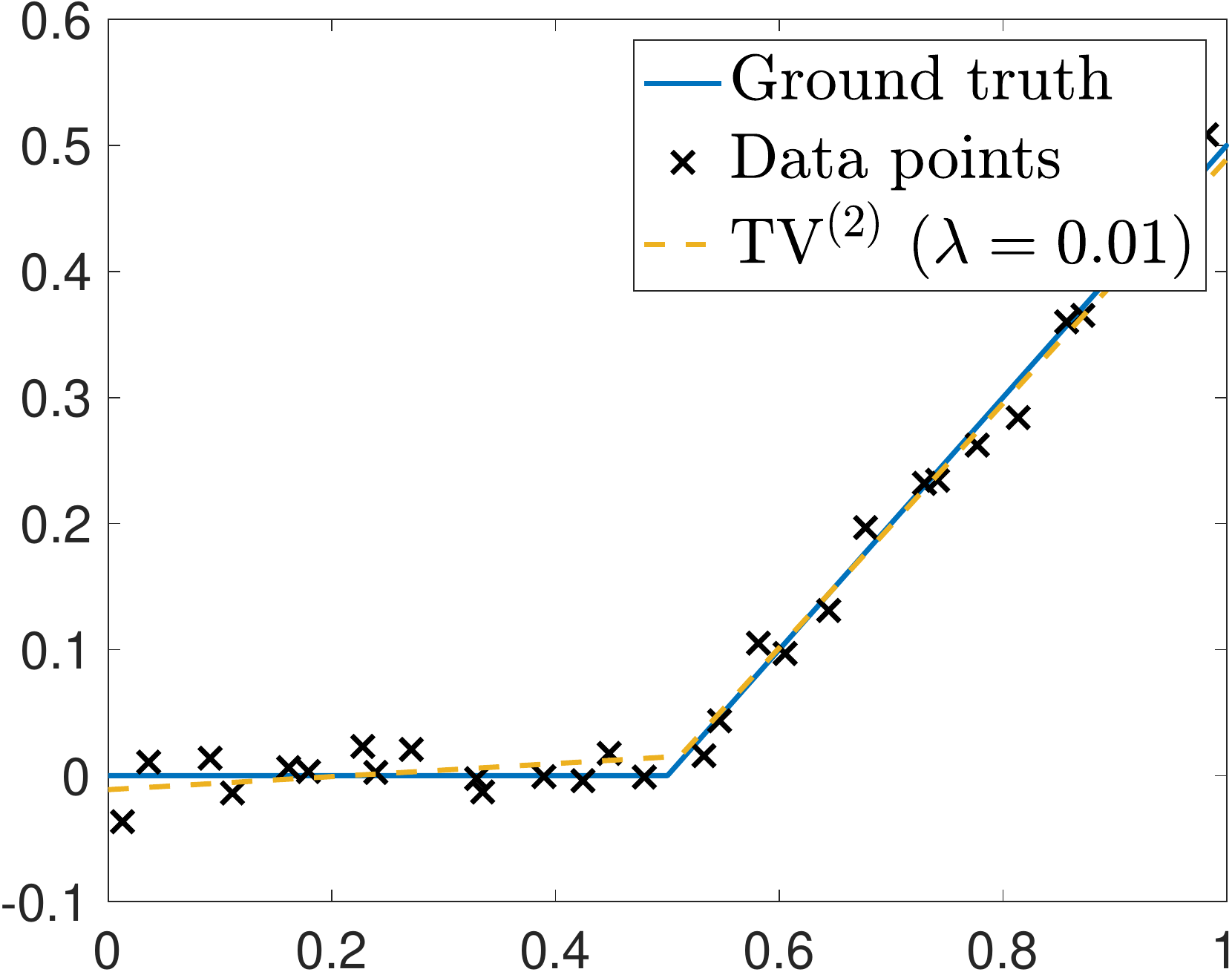}\label{fig:relu_TV2}}
    \caption{Reconstructions with a ReLU ground truth and $M=30$ data points. }
    \label{fig:limitation_lip}
\end{figure*}
 
 Hence, to reconstruct such a ground truth   accurately, it is necessary to enforce the sparsity of the reconstruction, which is exactly the purpose of the ${\rm TV}^{(2)}$ regularization. The reconstruction result of the ${\rm TV}^{(2)}$-regularized problem (\ie Problem~\eqref{Eq:Hybrid} with a relatively large Lipschitz bound) with $\lambda = 0.01$ is also shown in Figure~\ref{fig:relu_TV2}; it is clearly much more satisfactory than any of the Lipschitz-penalized reconstructions since it is very close to the ground truth and has the same sparsity (two linear regions).

 \subsection{Robustness to Outliers of the Lipschitz-Constrained Formulation}

 In this final experiment, we demonstrate the pertinence of our second formulation (Problem~\eqref{Eq:Hybrid}). More precisely, we examine  the increased robustness  to outliers of our second formulation~\eqref{Eq:Hybrid} with respect to ${\rm TV}^{(2)}$ regularization. To that end, we generate the CPWL ground truth $f_0$ with  6 linear regions and $M=50$ data points.  We then consider an additive Gaussian-noise model with low standard deviation $\sigma = 10^{-3}$ for  90\% of the data, and a much stronger $\sigma' = 3.5*10^{-2}$ for the remaining 10\%, which can be considered outliers.
 
    \begin{figure}[t]
     \centering
     \includegraphics[width=\linewidth]{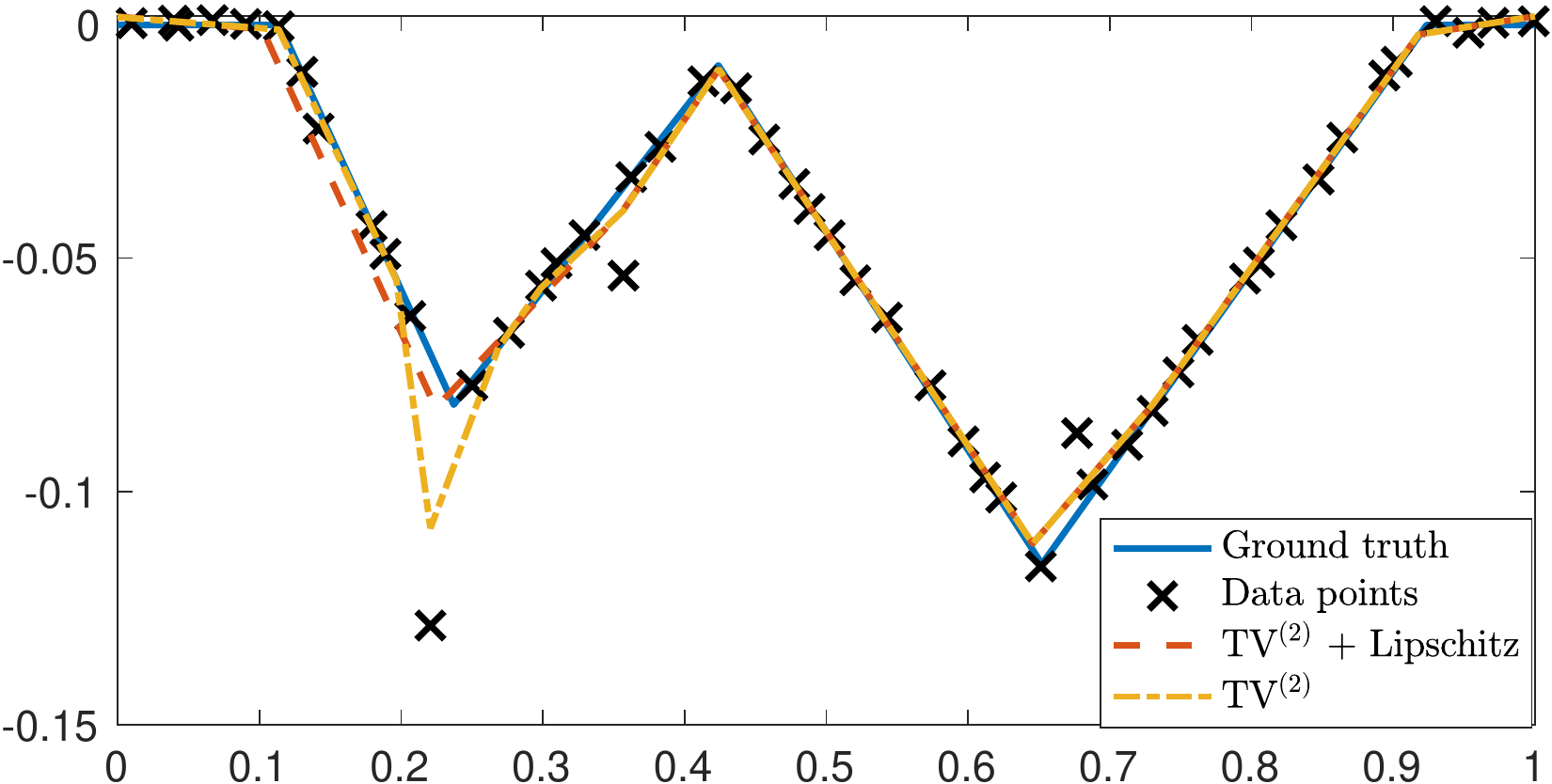}
       \caption{Reconstruction of $M=50$ data points for  $\lambda = 10^{-4}$. Our second formulation with $\bar{L} = 0.66$ produces 9 linear regions. We compare it to that of  ${\rm TV}^{(2)}$ which produces 12 linear regions.}
     \label{fig:TV2_vs_contraint}
 \end{figure}
 
 We show in Figure \ref{fig:TV2_vs_contraint} the reconstruction results using our second formulation with $\lambda = 10^{-4}$ and $\bar{L}  = 0.66$. The latter is quite satisfactory despite the presence of a strong outlier around $x_m = 0.22$. This is due to the fact that the Lipschitz constant is constrained. When using ${\rm TV}^{(2)}$-regularization alone, at  same regularization parameter, the reconstruction is very similar in most regions but is much more sensitive to this outlier which leads to an unwanted sharp peak and to the  high Lipschitz constant  $L(f^*) = 2.21$. Moreover, our reconstruction is more satisfactory in terms of sparsity (9 linear regions compared to 12, which is closer to the  $6$ linear regions of the target function $f_0$).  
 
 \section{Conclusion}
We have proposed two schemes for the learning  of one-dimensional continuous and piecewise-linear (CPWL) mappings with tunable Lipschitz constant. In the first scheme, we directly use the Lipschitz constant as a regularization term. We establish a representer theorem that allows us to deduce the existence of a CPWL solution for this continuous-domain optimization problem. In the second scheme, we use the second-order total-variation seminorm as the regularization term to which we add a Lipschitz constraint. Again, we proved the existence of a CPWL solution for this problem. Finally, we proposed an efficient algorithm to find the sparsest CPWL solution of each problem.  We illustrated the outcome of each scheme via numerical examples. A potential application of the proposed algorithm is to design stable CPWL activation functions with a minimum number of linear regions in deep neural networks.

\appendix
\subsection{Proof of Theorem \ref{thm:LipTV}}\label{App:LipTV}

\begin{proof}
For any $h>0$ and ${\bf p}=(p_1,p_2)\in\mathbb{R}^2$ with $p_1<p_2$, let us first define the test function $\varphi_h(\cdot;{\bf p})\in \mathcal{C}_0(\mathbb{R})$ as 
    \begin{align*}
    \varphi_h(x;{\bf p}) =h^{-1}\big(  &{\rm ReLU}\left(x- (p_1-h) \right) - {\rm ReLU}(x- p_1) \\&\quad + {\rm ReLU}\left(x-(p_2+h)\right) -{\rm ReLU}(x-p_2)\big).    
    \end{align*}
    
This function will be used on several occasions throughout  the proof. In particular, we use the explicit form of its second-order derivative given by 
\begin{align}
     {\rm D}^2\varphi_h(\cdot;{\bf p})  = h^{-1}\big(  &\delta\left(\cdot- (p_1-h) \right) - \delta(\cdot- p_1) \nonumber\\&\quad + \delta\left(\cdot-(p_2+h)\right) -\delta(\cdot-p_2)\big).    
    \label{Eq:phisec}
\end{align}

{\bf Upper-Bound:}
Similar to \eqref{Eq:LipSupD}, we have that ${\ell}(f) =\mathrm{ess} \inf_{x\in \mathbb{R}} |f'(x)|$. For a fixed $\epsilon>0$, by definition of the essential supremum and infimum, there exist   $\bar{x}, \underline{x}\in\mathbb{R}$ at which $f$ is differentiable with  
$|f'(\bar{x})| \geq \left(L(f) -\epsilon\right)$ and $|f'(\underline{x})| \leq \left(\ell(f) + \epsilon\right)$. Without   loss of generality, we  assume that $\bar{x}< \underline{x}$.  Following the limit definition of the derivative, we then consider a small radius $h>0$ such that 
\begin{align*}
   &\left| \frac{ f(\bar{x}+h)- f(\bar{x})}{h}\right| \geq  |f'(\bar{x})|- \epsilon \geq L(f) - 2\epsilon, \\&     \left|\frac{ f(\underline{x}+h)- f(\underline{x})}{h}\right| \leq  |f'(\underline{x})|+ \epsilon \leq \ell(f) + 2\epsilon.
\end{align*}
Now, let us  consider the test function $\varphi= \varphi_h\left(\cdot; (\bar{x}+h ,\underline{x}) \right)$.  Following the definition of the total-variation norm \eqref{Eq:TV} together with  $\|\varphi\|_{\infty}=1$, we deduce that ${\rm TV}^{(2)}(f)  \geq   |\langle {\rm D}^2 f, \varphi \rangle|  = |\langle  f, {\rm D}^2\varphi \rangle |,$
where the last equality  follows from   the self-adjointness of the second-order derivative. Using \eqref{Eq:phisec}, we thus have that
\begin{align*}
    {\rm TV}^{(2)}(f)& \geq h^{-1} |f(\bar{x}) - f(\bar{x}+h) + f(\underline{x}+h) - f(\underline{x})| \\& \geq \frac{|f(\bar{x}+h) - f(\bar{x})|}{h}  - \frac{|f(\underline{x}+h)- f(\underline{x})|}{h} \\&\geq L(f) -2\epsilon - \ell(f) - 2\epsilon 
    = L(f) - \ell(f) - 4\epsilon. 
\end{align*}
Finally, by letting $\epsilon \rightarrow 0$, we deduce the desired upper-bound.

{\bf Saturation---Sufficient Conditions:} Assume that $f\in {\rm BV}^{(2)}(\mathbb{R})$ is  convex and increasing; we  denote its second-order weak derivative by $w={\rm D}^2 f$. Note that, in this case, the functions $\left(-f(\cdot)\right)$, $f(-\cdot)$, and $\left(-f(-\cdot)\right)$ are concave/decreasing, convex/decreasing, and concave/increasing, respectively. Hence, we only need to prove the saturation for $f$ and the other cases   immediately  follow.  

For a fixed $\epsilon>0$, from  \eqref{Eq:TV2Def} there exists a test function $\psi \in \mathcal{D}(\mathbb{R})$ with compact support $K={\rm supp}(\psi)$ such that $\|\psi\|_{L_\infty}=1$ and 
    $\langle w, \psi \rangle \geq \left({\rm TV}^{(2)}(f) - \epsilon\right)$.
 For any $T>0$, we consider  the test function  $\psi_T = \varphi_1\left(\cdot; (-T,T)\right)$. From \eqref{Eq:phisec}, we obtain that 
\begin{align*}
    \langle w, \psi_T \rangle &= \langle f, {\rm D}^2 \psi_T \rangle \\&=\left(f(T+1)-f(T) \right) - \left(f(-T) - f(-T-1) \right)   \\&\leq   L(f) - \ell(f), 
\end{align*}
where we have used the increasing assumption to deduce that $f(T+1)\geq f(T)$ and $f(-T)\geq f(-T-1)$.
By choosing $T$ large enough so that $K\subseteq [-T,T]$, we ensure that $(\psi_T-\psi)$ is a nonnegative function, since for all $x\in K$, we will have that $\psi_T(x)=1=\|\psi\|_{L_{\infty}}\geq \psi(x)$. Next, the convexity of $f$ implies that   $w= {\rm D}^2 f$ is a  positive measure. Hence, 
\begin{equation}
    0\leq \langle w, \psi_T-\psi\rangle  \leq L(f) - \ell(f) - {\rm TV}^{(2)}(f) + \epsilon.
\end{equation}
By letting $\epsilon\rightarrow 0$, we deduce that ${\rm TV}^{(2)}(f) \leq \left(L(f) - \ell(f)\right)$, which  implies the saturation of \eqref{Eq:LipTV}. 

{\bf Saturation---Necessary Conditions:} Let $f\in {\rm BV}^{(2)}(\mathbb{R})$ be a function for which \eqref{Eq:LipTV} is saturated. 

{\bf Monotonicity:} Assume by contradiction that $f$ is not monotone. Hence, there exists $x_{n}\in \mathbb{R}$ such that  $f'(x_{n}) <0$.  Indeed, if $f'$ were a positive  distribution, then for any $a,b\in\mathbb{R}$ with $a<b$, we would have that 
$\left(f(b)- f(a)\right) = \langle f', \mathbbm{1}_{[a,b]} \rangle\geq 0,$
which contradicts the assumption of non-monotonicity.  Similarly, there exists $x_p\in\mathbb{R}$ such that $f'(x_p)>0$.

Next, consider a point $x_L \in \mathbb{R}$ such that $|f'(x_{L})| > \left(L(f) - \epsilon\right)>0$, 
where $0<\epsilon <\frac{\min(-f'(x_n),f'(x_p))}{3}$ is a small constant. Without   loss of generality, let us assume that $f'(x_L)>0$ and $x_n< x_L$. (For $f'(x_L)<0$, the same arguments can be applied to   $x_p$ instead of $x_n$.) There exists a small radius $h\in (0, \frac{x_L-x_n}{2})$ such that 
\begin{align*}
    &\frac{ f({x}_n+h)- f({x}_n)}{h} \leq f'(x_n) + \epsilon <0,        \\&\frac{ f({x}_L+h)- f({x}_L )}{h} \geq f'(x_L) -\epsilon >0.
\end{align*}
By considering the test function $\varphi= \varphi_h(\cdot; (x_n+h,x_L))$ and using \eqref{Eq:TV2Def} once again, we deduce that
\begin{align*}
    {\rm TV}^{(2)}(f) &\geq h^{-1} \left|f(x_n) - f(x_n+h) + f(x_L+h) - f(x_L)\right|   \\&= \frac{f(x_L+h) - f(x_L)}{h}  - \frac{f(x_n+h )- f({x}_n)}{h} \\& \geq f'(x_L) - \epsilon -f'(x_n) - \epsilon 
     \geq  L(f) - f'(x_n) - 3\epsilon 
      \\&> L(f),
\end{align*}
 which contradicts  the original assumption that \eqref{Eq:LipTV} is saturated. This proves that   $f$ is  monotone. In the following, we consider the case where $f$ is an increasing function;  the decreasing case can be deduced by symmetry. 

{\bf Convexity/Concavity:} We first consider the canonical decomposition  
$f = {\rm D}^{-2}_{\phi}w + p$,
where $w={\rm D}^2 f$, ${\rm D}^{-2}_{\phi}$ is a right inverse  of the second-order derivative,  and  $p(x)= ax+b$ is an affine term \cite[Proposition 9]{unser2019Deepspline}. We then 
use the Jordan decomposition of $w={\rm D}^2 f$ as $w= (w_+ - w_-)$, 
where $w_+, w_- \in\mathcal{M}(\mathbb{R})$ are positive measures such that $\|w\|_{\mathcal{M}} = \|w_+\|_\mathcal{M}+\|w_-\|_\mathcal{M}$. This  allows us to form the  decomposition  $f=\left(f_+ - f_-\right)$, where $f_s = {\rm D}^{-2}_{\phi}w_s + p_s, \quad s\in \{+,-\}$, 
$p_+(x)= (A+a)x +b$, and $p_-(x)= Ax$ with $A>0$ being a sufficiently large constant  such that the functions $f_+$ and $f_-$ are both convex and strictly increasing. Hence, they both satisfy the  sufficient conditions for saturation, which implies that ${\rm TV}^{(2)}(f_s) = \left(L(f_s)- \ell(f_s)\right)$ for $s\in \{+,-\}$. 

Let $\epsilon < \frac{\min({\rm TV}^{(2)}(f_+),{\rm TV}^{(2)}(f_-))}{2}$ be a small constant and let $\bar{x},\underline{x}\in\mathbb{R}$ such that $f'(\bar{x}) \geq \left(L(f) - \epsilon\right)$ and $f'(\underline{x}) \leq \left(\ell(f) +\epsilon\right)$. Using these inequalities,  we deduce that
\begin{align}
   {\rm TV}^{(2)}(f) &=  L(f) - \ell(f) \nonumber \\&\leq f'(\bar{x})  - f'(\underline{x})  + 2\epsilon  \nonumber\\&=  \left(f_+'(\bar{x}) - f_-'(\bar{x})\right)  - \left( f_+'(\underline{x}) - f_-'(\underline{x})\right) + 2\epsilon\nonumber \\&= A_+ - A_- + 2\epsilon, \label{Eq:TVupperApAn}
\end{align}
where $A_s = \left(f_s'(\bar{x}) - f_s'(\underline{x})  \right)$ for $s\in\{+,-\}$. Assume without  loss of generality  that $\bar{x}> \underline{x}$. The convexity of  $f_-$ implies that $A_- \geq 0$. Moreover, we have that 
$A_+ = \left(f_+'(\bar{x}) - f_+'(\underline{x})\right) \leq \left(L(f_+) - \ell(f_+)\right) = {\rm TV}^{(2)}(f_+)$.
 Using \eqref{Eq:TVupperApAn}, this yields that ${\rm TV}^{(2)}(f) \leq {\rm TV}^{(2)}(f_+) + 2\epsilon$, 
which can be rewritten as    $2\epsilon \geq {\rm TV}^{(2)}(f_-)$. This, together with our initial assumption on  $\epsilon$, yields that $w_{-}=0$, which implies that $f$ is   convex.
\end{proof}

\subsection{Proof of Theorem \ref{Thm:RepLipMin}} \label{App:RepLipMin}
\begin{proof}
{\bf Items \ref{ItemRepLipMin:Exist} and \ref{ItemRepLipMin:z}:}
The first step is to show that the sampling functional $\delta(\cdot-{x_0}):f\mapsto f(x_0)$ is weak*-continuous in ${\rm Lip}(\mathbb{R})$. To that end, we   identify the predual Banach space $\mathcal{X}$ such that ${\rm Lip}(\mathbb{R})= \mathcal{X}'$ and then show that shifted Dirac impulses are included in $\mathcal{X}$, which is equivalent to weak*-continuity. We recall that following   \eqref{Eq:LipNative}, we can view ${\rm Lip}(\mathbb{R})$ as the native Banach space associated to the pair $\left(L_{\infty}(\mathbb{R}),{\rm D}\right)$. This allows us to deploy the machinery of \cite{unser2019native} to identify its predual space. In short, it follows from \cite{unser2019native} that the predual space has the direct-sum structure $\mathcal{X}= {\rm D}\left(L_1(\mathbb{R}) \right) \oplus {\rm span}\left( {\rm e}^{-(\cdot)^2}\right)$.
 In other words, any function $f\in\mathcal{X}$ can be decomposed as $f= {\rm D}\{g\} + c{\rm e}^{-(\cdot)^2} $, where $g\in L_1(\mathbb{R})$ and $c\in\mathbb{R}$. One can formally verify that 
    $\delta = {\rm D}\{ {\rm sgn} -{\rm erf}\} +  \frac{2}{\sqrt{\pi}}{\rm e}^{-(\cdot)^2}$, 
where ${\rm sgn}$ is the sign function and ${\rm erf}$ is the Gauss error function. Due to the rapid decay of the erf function at $t=-\infty$ and the symmetry of $({\rm sgn}-{\rm erf})$, we deduce that ${\rm sgn}-{\rm erf}\in L_1(\mathbb{R})$ and, hence,  that $\delta \in \mathcal{X}$. Finally, due to the shift-invariant structure of $\mathcal{X}$, we   deduce the weak*-continuity of the sampling functional $\delta(\cdot-x_0)$ for any $x_0\in\mathbb{R}$. 

 Next, the powerful result of \cite{unser2020unifying} allows us to provide an abstract characterization of the solution set of \eqref{Eq:LipMin}. In particular, this ensures that the solution set $\mathcal{V}_{\rm Lip}$ of \eqref{Eq:LipMin} is a nonempty, convex, weak*-compact set whose elements all pass through a fixed set of points. Put differently, the vector ${\bf z}= (z_m)$ with $z_m=f(x_m)$ is invariant to the choice of $f\in \mathcal{V}_{\rm Lip}$. Consequently, we can represent $\mathcal{V}_{\rm Lip}$ as a solution set of a constrained problem of the form  \eqref{Eq:LipMinConst}.
 
{\bf Item \ref{ItemRepLipMin:LipOpt}:} Let us first define the canonical CPWL interpolant  of a collection of 1D data points. 
\begin{definition}\label{Def:CanInt}
For a series of data points $(x_m,z_m),m=1,\ldots,M$,   the canonical interpolant $f_{\rm cano}:\mathbb{R}\rightarrow\mathbb{R}$ is the unique CPWL function that passes through these points and is differentiable over $\mathbb{R}\backslash \{x_2,\ldots,x_{M-1}\}$. 
\end{definition}
 
We first prove that $f_{\rm cano}$ is a solution of \eqref{Eq:LipMinConst}. Clearly,   the Lipschitz constant of $f_{\rm cano}$ is equal to 
$L(f_{\rm cano}) = L_{\min}$,
where $L_{\min}$ is given in \eqref{Eq:LipOpt}. Moreover, any function $f$ that passes through the data points $(x_m,z_m)$ necessarily has a Lipschitz constant greater than or equal to $L_{\min}$. This implies that $f_{\rm cano}$ is a solution of \eqref{Eq:LipMinConst} and $L_{\min}$ is the minimal value of the Lipschitz constant. Consequently, any function that satisfies the interpolation constraints and is $L_{\min}$-Lipschitz is a solution of \eqref{Eq:LipMinConst}. 

{\bf Item \ref{ItemRepLipMin:Graph}:} Consider a generic point $(x,y)\in \mathcal{E}$, and let $m$ be such that $x\in (x_{m-1},x_m)$. By definition of $\mathcal{E}$, there exists a function $f\in \mathcal{V}_{\rm Lip}$ such that $y=f(x)$. From Item 3, we deduce that $L(f)=L_{\min}$. Hence, we have the inequalities 
\begin{equation}\label{Eq:IneqLR}
\left|\frac{y- z_{m-1}}{x-x_{m-1}}\right| ,  \left|\frac{y- z_{m}}{x-x_{m}}\right|  \leq L_{\min}.
\end{equation}
These inequalities can readily be translated into the inclusion $(x,y)\in \mathcal{R}_{m-1} \cap \mathcal{L}_{m}$, which implies that 
$\mathcal{E}\subseteq  \bigcup_{m=1}^M \left(\mathcal{R}_{m-1}\cap \mathcal{L}_m\right)$.
To show the reverse inclusion, consider a point in $(x,y)\in \mathcal{R}_{m-1} \cap \mathcal{L}_m$  for some $m\in\{1,\ldots,M+1\}$ and denote by $\tilde{f}_{\rm cano}$ the canonical interpolant of $\{(x_m,z_m)\}_{m=1}^M\cup \{(x,y)\}$. Following   Item \ref{ItemRepLipMin:LipOpt}, the Lipschitz constant of $ \tilde{f}_{\rm cano}$  is given by
\begin{equation}
L(\tilde{f}_{\rm cano}) = \max\left(L_{\min}, \left|\frac{y-z_{m-1}}{x-x_{m-1}}\right|,\left|\frac{y-z_{m}}{x-x_{m}}\right|\right)  = L_{\min},
\end{equation}
where we establish the last equality by translating the inclusion $(x,y)\in \mathcal{R}_{m-1} \cap \mathcal{L}_m$ into the inequalities   in \eqref{Eq:IneqLR}.  This    implies that $\tilde{f}_{\rm cano}$ is a solution of \eqref{Eq:LipMinConst} and so, by definition, we   have  that $(x,y)\in \mathcal{E}$.

{\bf Item \ref{ItemRepLipMin:TV2}:} By \cite[Proposition 5]{debarre2020sparsest}, $f_{\rm cano}$ is also a solution of \eqref{Eq:TV2MinConst}.  We therefore need to prove that any solution $f_{\rm opt}$ of 
\eqref{Eq:TV2MinConst} has the same Lipschitz constant $L(f_{\rm 
opt}) = L(f_{\rm cano})=L_{\min}$. Due to the interpolation constraints, we 
necessarily have that $L(f_{\rm opt}) \geq L(f_{\rm cano})$; we must now prove 
the reverse inequality $L(f_{\rm opt}) \leq L(f_{\rm cano})$. By 
\cite[Theorem 2]{debarre2020sparsest}, $f_{\rm opt}$ must follow 
$f_{\rm cano}$ in $\mathbb{R} \backslash [ x_2, x_{M-1} ]$. Moreover, in each interval 
$[x_m, x_{m+1}]$ for $m \in \{2, \ldots , M-2 \}$, $f_{\rm opt}$ either 
follows $f_{\rm cano}$  or is concave or convex over the interval $
[x_{m-1}, x_{m+2}]$. Hence, it suffices to prove that, for any $m \in \{2, 
\ldots , M-2 \}$, we have that $L_m(f_{\rm opt}) \leq L(f_{\rm cano})$, 
where $L_m(f)$ denotes the Lipschitz constant of $f$ restricted to the 
interval $[x_{m}, x_{m+1}]$.

Let $m$ be an index for which $f_{\rm opt}$ need not follow $f_{\rm opt}$ in $[x_{m}, x_{m+1}]$.  (If no such index exists, then the result is trivially true.) Assume that $f_{\rm opt}$ is convex in the interval $[x_{m-1}, x_{m+2}]$; the concave scenario is derived in a similar fashion. This implies that, in this interval, the function $(\tilde{x}_1, \tilde{x}_2) \mapsto \frac{f_{\rm opt}(\tilde{x}_2) - f_{\rm opt}(\tilde{x}_1)}{\tilde{x}_2 - \tilde{x}_1}$ is increasing in both its variables.

Hence, for any $\tilde{x}_1, \tilde{x}_2 \in [x_m, x_{m+1}]$ with $\tilde{x}_1 \neq \tilde{x}_2$, we have that $\frac{z_m - z_{m-1}}{x_m-x_{m-1}} \leq \frac{f_{\rm opt}(\tilde{x}_2) - f_{\rm opt}(\tilde{x}_1)}{\tilde{x}_2 - \tilde{x}_1} \leq \frac{z_{m+2} - z_{m+1}}{x_{m+2}-x_{m+1}}$. This directly implies the desired result $L_m(f_{\rm opt}) \leq L(f_{\rm cano})$. 
\end{proof}

\section{Proof of Theorem \ref{Thm:DebarreOptimal}}\label{App:DebarreOptimal}
\begin{proof}

Let $f^\ast$ be the output of \cite[Algorithm 1]{debarre2020sparsest}. It is thus a CPWL solution of Problem~\eqref{Eq:LipMinConst} with the minimum number of linear regions. We prove that \emph{any} CPWL interpolant $f$ of the data points $\mathrm{P}_m = (x_m, z_m), m=1,\ldots,M$---not necessarily a minimizer of ${\rm TV}^{(2)}(f)$---has at least as many linear regions as $f^\ast$. Our proof is based on induction over the number  $M$ of data points. The initialization $M=2$ trivially holds, since $f^\ast$ then has a single linear region---it is simply the line connecting the two data points. Next, let $M > 2$ and assume that Theorem~\ref{Thm:DebarreOptimal} holds for $(M-1)$ or less data points (the induction hypothesis). The canonical interpolatant $\fcano$ introduced in Definition \ref{Def:CanInt} can be expressed as  
\begin{equation}
    \label{eq:canonical_sol}
        \fcano(x) = \alpha_1 x + \alpha_2 + \sum_{m=2}^{M-1} a_m (x - x_m)_+
    \end{equation}
    for some coefficients $\alpha_1,\alpha_2,a_m\in\mathbb{R}$.
There are three possible scenarios: 
\begin{enumerate}
\item all $a_m$'s are positive (or negative);\label{Scenario1}
\item at least one of them is zero;\label{Scenario2}
\item there are two consecutive coefficients with opposite signs, so that $a_m a_{m+1}<0$ for some $m$. \label{Scenario3}
\end{enumerate} 
We analyze each case separately and use the induction hypothesis to deduce the desired result. In this proof, we refer to singularities of CPWL functions (\ie the boundary points between linear regions) as \emph{knots}.

{\bf Case \ref{Scenario1}:} In this case, it is known that $f^*$ has $K= \left(\lceil \frac M2 \rceil - 1\right)$ knots \cite[Theorem 4]{debarre2020sparsest}. Assume by contradiction that there exists a CPWL interpolant $f$ with fewer knots and consider the $K$ disjoint intervals $(x_{2k-1}, x_{2k+1})$ for $1 \leq k \leq \left(\lceil \frac M2 \rceil - 1 \right)= K$. We deduce that  there exists an interval  $(x_{2k-1}, x_{2k+1})$ in which $f$ has no knots. This in turn implies that the data points $\mathrm{P}_{2k-1}$, $\mathrm{P}_{2k}$, and $\mathrm{P}_{2k+1}$ are aligned, and so that $a_{2k} = 0$, which yields a contradiction.

{\bf Case \ref{Scenario2}:} Let $m \in \{2, M-1 \}$ be such that $a_m = 0$. Consider the collection of $m < M$ data points $(\mathrm{P}_{m'})_{1 \leq m' \leq m}$; by the induction hypothesis, $f^\ast$ interpolates them with the minimal number $K_1$ of knots. The same applies to the collection of $\left(M-m+1\right) < M$ points $(\mathrm{P}_{m'})_{m \leq m' \leq M}$ with $K_2$ knots. Let $f$ be a CPWL interpolant of all the $M$ data points with the minimal number of knots. By definition of the $K_i$, $f$ must have at least $K_1$ knots in the interval $(x_1, x_{m})$ and $K_2$ knots in the interval $(x_{m}, x_{M})$. Since these intervals are disjoint, $f$ must have at least $K_1 + K_2$ knots in total. Yet, $f^\ast$ has exactly $(K_1 + K_2)$ knots: indeed, $f^\ast$ follows $\fcano$ in the interval $[x_{m-1}, x_{m+1}]$, which has no knot at $x_m$ since $a_m = 0$ (the points $\mathrm{P}_{m-1}$, $\mathrm{P}_{m}$, and $\mathrm{P}_{m+1}$ are aligned). This concludes that $f^\ast$ has the minimum number of knots.

{\bf Case \ref{Scenario3}:} Let  $m \in \{2, M-2 \}$ be such that $a_m a_{m+1} < 0$. Consider the collection of $(m+1) < M$ data points $(\mathrm{P}_{m'})_{1 \leq m' \leq m+1}$; by the induction hypothesis, $f^\ast$ interpolates them with the minimal number $K_1$ of knots. Similarly, $f^\ast$ interpolates the $(M-m+1) < M$ points $(\mathrm{P}_{m'})_{m \leq m' \leq M}$ with the minimal number $K_2$ of knots. Let $f$ be a CPWL interpolant of all the $M$ data points with the minimal number of knots. We now state a useful lemma whose proof is given below.  
\begin{lemma}
\label{lem:knots_overlaps}
Let $m \in \{2, \ldots , M-2 \}$ be such that $a_m a_{m+1} < 0$. Then, any CPWL interpolant $f$ of the data points $(\mathrm{P}_{m'})_{1 \leq m' \leq M}$     can be modified to become another CPWL interpolant $\tilde{f}$ with as many (or fewer) knots  such that $\tilde{f}$ has no knot in the interval $(x_m, x_{m+1})$. 
\end{lemma}
By Lemma~\ref{lem:knots_overlaps}, it can be modified to become another interpolant $\tilde{f}$ with the same total number of knots and none in the interval $(x_m, x_{m+1})$. By definition of the $K_i$, $\tilde{f}$ must have at least $K_1$ knots in the interval $(x_1, x_{m+1})$ and $K_2$ knots in the interval $(x_{m}, x_{M})$. Yet, $\tilde{f}$ has no knots in the interval $(x_m, x_{m+1})$, so it must have at least $K_1$ knots in $(x_1, x_{m}]$ and $K_2$ knots in $[x_{m+1}, x_{M})$. Since these intervals are disjoint, $\tilde{f}$ must have at least $(K_1 + K_2)$ knots in total. Yet, $f^\ast$ follows $\fcano$ in the interval $[x_{m-1}, x_{m+2}]$ and thus also has no knot in the interval $(x_1, x_{m+1})$. Therefore, by the induction hypothesis, $f^\ast$ has $K_1$ knots in $(x_1, x_{m}]$ and $K_2$ knots in $[x_{m+1}, x_{M})$, for a total of $(K_1 + K_2)$ knots. Since this is no more than $\tilde{f}$, $f^\ast$ has the minimal number of knots, which proves the induction.
\end{proof}

\begin{proof}[Proof of Lemma \ref{lem:knots_overlaps}]
 Let $f$ be a CPWL interpolant of the data points $(\mathrm{P}_{m'})_{1 \leq m' \leq M}$ with $P$ knots. In what follows, we consider a CPWL function $\tilde{f}$ that follows $f$ outside this interval and $(x_{m-1}, x_{m+2})$, and we modify it inside this interval in order to remove all knots in $(x_m, x_{m+1})$ without increasing the total number of knots.
 
 We consider the case $a_m > 0$ and $a_{m+1} < 0$ without   loss of generality. Let $s^- = f'(x_{m-1}^-)$ and $s^+ = f'(x_{m+2}^+)$ be the slopes of $f$ before and after the interval of interest $(x_{m-1}, x_{m+2})$, respectively, and we let $s_{\mathrm{cano}}^- = \fcano'(x_{m-1}^-)$ and $s_{\mathrm{cano}}^+ = \fcano'(x_{m+2}^+)$ be those of $\fcano$. We also introduce the linear functions $f^-(x) = z_{m-1} + s^- (x - x_{m-1})$ and $f^+(x) = z_{m+2} + s^+ (x - x_{m+2})$. They prolong $f$ in a straight line after $\mathrm{P}_{m-1}$ and before $\mathrm{P}_{m+2}$, respectively. We now distinguish cases based on $s^-$ and $s^+$. 
  
 {\bf Case I:  $s^- \leq s_{\mathrm{cano}}^-$ and $s^+ \leq s_{\mathrm{cano}}^+$.}
Graphically, this corresponds to  $f$ lying in none of the gray regions in Figure~\ref{fig:overlaps}.  In this case, the line $(\mathrm{P}_{m} \mathrm{P}_{m+1})$ intersects the linear function $f^-$ at some point $\mathrm{P}^- = (x^-, z^-)$ where $x^- \in (x_{m-1}, x_{m})$, and with $f^+$ at some point $\mathrm{P}^+ = (x^+, z^+)$ with $x^+ \in (x_{m+1}, x_{m+2})$. This is obvious graphically (see Figure~\ref{fig:overlaps} as an illustration for $\mathrm{P}^-$), and is due to the fact that $a_m > 0$ and $a_{m+1} < 0$. Hence, by taking an $\tilde{f}$ that connects the points $\mathrm{P}_{m-1}$, $\mathrm{P}
^-$, $\mathrm{P}^+$, and $\mathrm{P}_{m+2}$, then $\tilde{f}$ has two knots in $[x_{m-1}, x_{m+2}]$ and its knots satisfy $x^-, x^+ \not\in (x_m, x_{m+1})$. Since $f$ clearly cannot have fewer than two knots in this interval, this proves the desired result.

{\bf Case II:  $s^+ > s_{\mathrm{cano}}^+$ and $s^- > s_{\mathrm{cano}}^-$.}  In this case,  $f$ lies in both gray regions in Figure~\ref{fig:overlaps}. To pass through $\mathrm{P}_m$, $f$ must have at least one knot in $[x_{m-1}, x_m)$; let $\mathrm{P}^- = (x^-, z^-)$ be the first of those knots (with $x^- < x_m$). Similarly,   to pass through $\mathrm{P}_{m+1}$, $f$ must have a knot in $(x_{m+1}, x_{m+2}]$; let $\mathrm{P}^+ = (x^+, z^+)$ be the last of those knots (with $x^+ > x_{m+1}$).  Then, $f$ must pass through the points $\mathrm{P}^-$, $\mathrm{P}_{m}$, $\mathrm{P}_{m+1}$, $\mathrm{P}^+$. Yet, the lines $(\mathrm{P}^- \mathrm{P}_{m})$ and $(\mathrm{P}_{m+1} \mathrm{P}^+)$ clearly cannot intersect in the interval $[x_m, x_{m+1}]$, which implies that at least two knots are needed in the interval $(x^-, x^+)$. We conclude that $f$ must have at least four knots in the interval $[x_{m-1}, x_{m+2}]$. Hence, we  take an $\tilde{f}$ that simply connects the points $\mathrm{P}_{m-1}$, $\mathrm{P}_{m}$, $\mathrm{P}_{m+1}$, and $\mathrm{P}_{m+2}$ and follows $f$ elsewhere; the latter has four knots in $[x_{m-1}, x_{m+2}]$, which is no more than $f$ and thus fulfills the requirements of the proof.

{\bf Case III:  $s^+ > s_{\mathrm{cano}}^+$ and $s^- \leq s_{\mathrm{cano}}^-$}. This case is illustrated in Figure~\ref{fig:overlaps}: $f$ is outside the gray region on the left, and inside the one on the right. With a similar argument as in Case II, $f$ must have a least three knots in the interval $[x_{m-1}, x_{m+2}]$. The fact that $a_m > 0$ implies that the line $(\mathrm{P}_{m} \mathrm{P}_{m+1})$ intersects the linear function $f^-$ at some point $\mathrm{P}^- = (x^-, z^-)$ where $x^- \in (x_{m-1}, x_{m})$. We then take an $\tilde{f}$ that connects the points $\mathrm{P}_{m-1}$, $\mathrm{P}
^-$, $\mathrm{P}_{m+1}$, and $\mathrm{P}_{m+2}$ and follows $f$ elsewhere. The interpolant $\tilde{f}$ has three knots at $x^-$, $x_{m+1}$, and $x_{m+2}$ in $[x_{m-1}, x_{m+2}]$ and thus satisfies the requirements of the proof.

{\bf Case IV: 
$s^+ \leq s_{\mathrm{cano}}^+$ and $s^- > s_{\mathrm{cano}}^-$.} This  is similar to Case III, and  can be readily deduced by symmetry, thus completing the proof of Lemma \ref{lem:knots_overlaps}.

\begin{figure}
    \centering
\includegraphics[width=\linewidth]{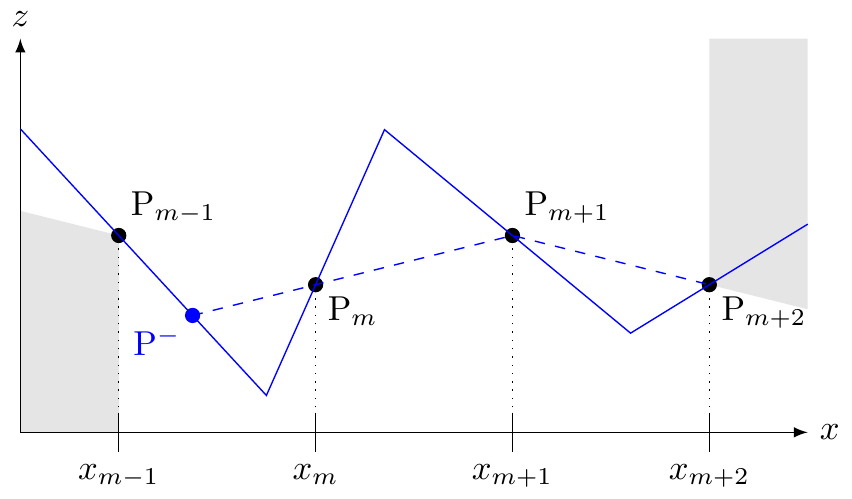}
    \caption{Illustration of Lemma~\ref{lem:knots_overlaps} in the case $a_m > 0$ and $a_{m+1} < 0$. The interpolant $f$ (solid line) satisfies $s^+ > s_{\mathrm{cano}}^+$  and $s^- \leq s_{\mathrm{cano}}^-$. The modified interpolant $\tilde{f}$ (dashed line) also has three knots $\mathrm{P}^-$, $\mathrm{P}_{m+1}$, and $\mathrm{P}_{m+2}$, but none in $(x_m, x_{m+1})$.}
    \label{fig:overlaps}
\end{figure}
\end{proof}

\subsection{Proof of Theorem \ref{Thm:Hybrid}}\label{App:Hybrid}
   \begin{proof}
 {\bf Existence:} We rewrite the problem in \eqref{Eq:Hybrid} as an unconstrained minimization  
 \begin{equation}\label{Eq:HybridUnconst}
     \mathcal{V}_{\rm hyb}=\argmin_{f\in \mathcal{M}_{{\rm D}^2}(\mathbb{R})} \sum_{m=1}^M {\rm E}(f(x_m),y_m) + \lambda {\rm TV}^{(2)}(f) + i_{L(f)\leq \bar{L}},
 \end{equation}
 where $i_E$ denotes the characteristic function of the set $E$ and is defined as 
 \begin{equation}
     i_E(f) = \begin{cases} 0, \quad f\in E \\ 
     +\infty, \quad \text{otherwise.}
     \end{cases}
 \end{equation}
 To prove the existence of a minimizer, we use a standard technique in convex analysis which involves the generalized Weierstrass theorem \cite{kurdila2006convex} to show that the cost functional of \eqref{Eq:HybridUnconst} is coercive and lower semicontinuous (in the weak*-topology), which is a sufficient condition for the existence of a solution.

 The cost functional in \eqref{Eq:Hybrid}  consists of three terms: (i) an empirical loss term $H(f) = \sum_{m=1}^M {\rm E}(f(x_m),y_m)$; (ii) a second-order total-variation regularization term $R(f) = \lambda {\rm TV}^{(2)}(f)$; and (iii) a Lipschitz constraint $i_E$, where $E=\{L(f)\leq \bar{L}\}$. It is known (see \cite{Gupta2018gTV} for a more general statement) that the functional $H(f)+ R(f)$ is coercive and weak*-lowersemincontinuous. This, together with the non-negativity of $i_E$, yields the   coercivity of the total cost. The only missing item is the weak*-lowersemicontinuity of $i_E$, for which  it is sufficient to  prove that $E$ is a closed set for  the  weak*-topology. 
 
 Let $f_n \in {\rm BV}^{(2)}(\mathbb{R})$ be a sequence of functions with $L(f_n) \leq \bar{L}$ converging in the weak*-topology to $f_{\lim} \in {\rm BV}^{(2)}(\mathbb{R})$. To prove the weak*-closedness of $E$, we need to show that $L(f_{\lim}) \leq \bar{L}$, which is   equivalent to $| f_{\lim}(a)- f_{\lim}(b)| \leq \bar{L}|a-b|$ for any   $a,b\in\mathbb{R}$. 
 
 For any $n\in \mathbb{N}$, we have that 
\begin{align}
     | f_{\lim}(a)- f_{\lim}(b)|  &\leq  |f_{\lim}(a)- f_n(a) |+|f_n(a) - f_n(b)|\nonumber \\& \qquad+ |f_n(b) - f_{\lim}(b)|.\label{Eq:flimLip}
\end{align}
Using the weak*-continuity of the sampling functionals $\delta(\cdot-a)$ and $\delta(\cdot-b)$ in ${\rm BV}^{(2)}(\mathbb{R})$ (see, for example,  \cite{unser2019Deepspline}), we deduce that 
    $f_n(a) \rightarrow f_{\lim}(a)$ and  $f_n(b) \rightarrow f_{\lim}(b)$. Moreover, we have the estimate $|f_n(a) - f_n(b)|\leq \bar{L} |a-b|$ for any $n\in\mathbb{N}$. Using these and  letting the limit $n\rightarrow +\infty$ in  \eqref{Eq:flimLip}, we get the desired bound.

{\bf Form of the Solution Set:} Now that we have proved the existence of a solution $f_0$, we can apply a standard argument based on the strict convexity of $E(\cdot,\cdot)$ (see, for example, \cite[Lemma 1]{tibshirani2013lasso}) to deduce that the original problem \eqref{Eq:Hybrid} is equivalent to 

\begin{align}
    \mathcal{V}_{\rm hyb}=\argmin_{f\in {\rm BV}^{(2)}(\mathbb{R})}  &{\rm TV}^{(2)}(f), \nonumber\\& \text{s.t.} \quad \begin{cases} L(f)\leq \bar{L}, & \\  f(x_m) = f_0(x_m),  & m=1,\ldots,M.\end{cases}\label{Eq:HybridReduced} 
 \end{align}
 Yet, Theorem \ref{Thm:RepLipMin} implies that   the constraint $L(f)\leq \bar{L}$ is automatically satisfied for any solution of 
 \begin{align}
    \mathcal{V}_{\rm hyb}=\argmin_{f\in  {\rm BV}^{(2)}(\mathbb{R})}  &{\rm TV}^{(2)}(f), \nonumber \\& \text{s.t.} \quad   f(x_m) = f_0(x_m), m=1,\ldots,M.\label{Eq:HybridReduced2}
 \end{align}
Hence, the original problem \eqref{Eq:Hybrid} is equivalent to \eqref{Eq:HybridReduced2} whose solution set has been fully described in   \cite{debarre2020sparsest}. 
 \end{proof}
 
 \subsection{Proof of Proposition \ref{prop:equivalence_training_NN}} \label{App:trainingNN}
  
 We start by proving a useful lemma. 
 \begin{lemma}\label{lem:optimalNN}
 For any $\boldsymbol{\theta}^*=(K^*,\mathbf{v}^*,\mathbf{w}^*,\mathbf{b}^*,\mathbf{c}^*) \in \mathcal{V}_{\rm NN}$, we have that $|v^*_k|=|w^*_k|$ for any $k=1,\ldots,K$.
 \end{lemma}
 \begin{proof}

 Let $\boldsymbol{\theta}^*=(K^*,\mathbf{v}^*,\mathbf{w}^*,\mathbf{b}^*,\mathbf{c}^*) \in \mathcal{V}_{NN}$ and $1 \leq k \leq K$. For any $\epsilon\in (-1,1)$, we define a perturbed parameter vector $\boldsymbol{\theta}_{\epsilon}=(K^*,\mathbf{v}_\epsilon,\mathbf{w}_\epsilon,\mathbf{b}_\epsilon,\mathbf{c}^*)$, where  for any $k'=1,\ldots,K$ we have that 
 \begin{align}
     v_{\epsilon,k'}&= \begin{cases} v^*_{k'}, & k'\neq k \\ (1+\epsilon)^{\frac{1}{2}} v^*_k, & k'=k\end{cases},  \\ 
     w_{\epsilon,k'}&= \begin{cases} w^*_{k'}, & k'\neq k \\ (1+\epsilon)^{-\frac{1}{2}} w^*_k, & k'=k\end{cases},\\
     b_{\epsilon,k'}&= \begin{cases} b^*_{k'}, & k'\neq k \\ (1+\epsilon)^{-\frac{1}{2}} b^*_k, & k'=k.\end{cases}
 \end{align}

Due to the positive homogeneity of the ReLU, one readily deduces from \eqref{Eq:TwoLayerNN} that  $f_{\boldsymbol{\theta}^*} = f_{\boldsymbol{\theta}_\epsilon}$ for any $\epsilon \in (-1,1)$. This together with the optimality of $\boldsymbol{\theta}^*$ in Problem~\eqref{Eq:NNlearning} implies that 
 $${v_k^*}^2 + {w_k^*}^2 \leq (1+\epsilon) {v_k^*}^2 + (1+\epsilon)^{-1} {w_k^*}^2, \quad \forall \epsilon \in (-1,1).$$ 
Multiplying both sides of the above inequality by $(1 + \epsilon) > 0$ yields
 $$  \epsilon {w_k^*}^2 \leq \epsilon(1+\epsilon) {v_k^*}^2, \quad \forall \epsilon \in (-1,1).$$ 
Letting $\epsilon \to 0^+$ yields ${w_k^*}^2 \leq {v_k^*}^2$ and $\epsilon \to 0^-$ yields ${w_k^*}^2 \geq {v_k^*}^2$, which proves that $|w_k^*|=|v_k^*|$.
\end{proof}
\begin{proof}[Proof of Proposition \ref{prop:equivalence_training_NN}]
Using Lemma \ref{lem:optimalNN}, we observe that for any $\boldsymbol{\theta}^*\in \mathcal{V}_{\rm NN}$, we have that 
$$ {\rm R}(\boldsymbol{\theta}^*) = \frac{1}{2} \sum_{k=1}^K ({v_k^*}^2 + {w_k^*}^2) = \sum_{k=1}^K |v_k^*||w_k^*| = {\rm TV}^{(2)}(f_{\boldsymbol{\theta}^*}),  $$
where the last inequality comes from the simple observation that ${\rm TV}^{(2)}( v{\rm ReLU}(w \cdot - b)) = |v||w|$ for any $v,w,b\in\mathbb{R}$.
Hence, one can rewrite the solution set $\mathcal{V}_{\rm NN}$ as 
\begin{align*}
 \mathcal{V}_{\rm NN} = \argmin_{\boldsymbol{\theta}\in\boldsymbol{\Theta}_{\rm red}} &\left(\sum_{m=1}^M {\rm E}(f_{\boldsymbol{\theta}}(x_m),y_m) + \lambda {\rm TV}^{(2)}(f_{\boldsymbol{\theta}})\right), \\& \text{s.t.} \quad L(f_{\boldsymbol{\theta}})\leq \bar{L},
\end{align*}
where $\boldsymbol{\Theta}_{\rm red}=\{\boldsymbol{\theta}\in \boldsymbol{\Theta}: {\rm R}(\boldsymbol{\theta})= {\rm TV}^{(2)}(f_{\boldsymbol{\theta}})\}$ is the reduced parameter space. To prove the announced equivalence, it remains to show that the mapping  $\boldsymbol{\Theta}_{\rm red}\rightarrow {\rm BV}^{(2)}(\mathbb{R}):\boldsymbol{\theta}\mapsto f_{\boldsymbol{\theta}}$ is a bijection onto the CPWL members of ${\rm BV}^{(2)}(\mathbb{R})$ with finitely many linear regions.

For any $\boldsymbol{\theta}\in\boldsymbol{\Theta}_{\rm red}$, the function $f_{\boldsymbol{\theta}}$ is a CPWL member of ${\rm BV}^{(2)}(\mathbb{R})$ with finitely many linear regions. To prove the converse, let    $f\in{\rm BV}^{(2)}(\mathbb{R})$ be a CPWL function with finitely many linear regions. Using the canonical representation of $f$, there exist $c_0,c_1\in \mathbb{R}$, $K\in\mathbb{N}$ and $a_k,\tau_k \in\mathbb{R}$ with $a_k\neq 0$ for $k=1,\ldots, K$ such that 
$$ f(x) = c_0 + c_1 x + \sum_{k=1}^K a_k {\rm ReLU}(x-\tau_k).$$
Now by defining $v_k = \frac{a_k}{\sqrt{|a_k|}}$, $w_k = \sqrt{|a_k|}$ and, $b_k = \sqrt{|a_k|}\tau_k$ for $k=1,\ldots,K$, the homogeneity of the ReLU yields $f=f_{\boldsymbol{\theta}}$ with $\theta=(K,{\bf c}, {\bf v}, {\bf w},{\bf b})\in \boldsymbol{\Theta}_{\rm red}$, where the latter inclusion is due to the equalities $|v_k|=|w_k|$ for $k=1,\ldots,K$.
\end{proof}
\bibliographystyle{IEEEtran}
\bibliography{RefLipTV}

\end{document}